\documentclass{article}

\PassOptionsToPackage{numbers, compress}{natbib}
\usepackage[final]{neurips_2018}

\usepackage[utf8]{inputenc} %
\usepackage[T1]{fontenc}    %
\usepackage{hyperref}       %
\usepackage{url}            %
\usepackage{booktabs}       %
\usepackage{amsfonts}       %
\usepackage{nicefrac}       %
\usepackage{microtype}      %

\usepackage{times}
\usepackage{graphicx} %
\usepackage{algorithm}
\usepackage{algorithmic}

\usepackage{amsmath, amssymb, amsthm, bm, color}
\usepackage{enumitem}
\usepackage{notation}
\usepackage{times}
\usepackage{subcaption}
\usepackage{booktabs}       %
\usepackage{array}
\usepackage{multirow}
\usepackage[table,xcdraw]{xcolor}
\usepackage{mathtools}
\usepackage{multicol}

\usepackage{calc} %
\usepackage{wrapfig}

\usepackage{floatflt}
\definecolor{mydarkblue}{rgb}{0,0.08,0.45}

\hypersetup{ %
    pdftitle={Quantifying Learning Guarantees for Convex but Inconsistent Surrogates},
    pdfauthor={Kirill Struminsky, Simon Lacoste-Julien, Anton Osokin},
    pdfsubject={Quantifying Learning Guarantees for Convex but Inconsistent Surrogates},
    pdfkeywords={structured prediction, consistency, surrogate losses, machine learning},
    pdfborder=0 0 0,
    pdfpagemode=UseNone,
    colorlinks=true,
    linkcolor=mydarkblue,
    citecolor=mydarkblue,
    filecolor=mydarkblue,
    urlcolor=mydarkblue,
    pdfview=FitH
}

\title{Quantifying Learning Guarantees for Convex but Inconsistent Surrogates}

\author{
    Kirill Struminsky\\
    NRU HSE,\thanks{National Research University Higher School of Economics}\; Moscow, Russia\\
   \And
   Simon Lacoste-Julien\thanks{CIFAR Fellow} \\
    MILA and DIRO \\
   Universit\'{e} de Montr\'{e}al, Canada
   \And
   Anton Osokin\\
   NRU HSE,\footnotemark[1]\:\,\thanks{Samsung-HSE Joint Lab}\; Moscow, Russia\\
   Skoltech,\thanks{Skolkovo Institute of Science and Technology}\; Moscow, Russia
}

\begin{document}

\maketitle

\begin{abstract}
We study consistency properties of machine learning methods based on minimizing convex surrogates. We extend the recent framework of~\citet{osokin17consistency} for the quantitative analysis of consistency properties to the case of inconsistent surrogates. Our key technical contribution consists in a new lower bound on the calibration function for the quadratic surrogate, which is non-trivial (not always zero) for inconsistent cases. The new bound allows to quantify the level of inconsistency of the setting and shows how learning with inconsistent surrogates can have guarantees on sample complexity and optimization difficulty. We apply our theory to two concrete cases: multi-class classification with the tree-structured loss and ranking with the mean average precision loss. The results show the approximation-computation trade-offs caused by inconsistent surrogates and their potential benefits.
\end{abstract}

\section{Introduction} \label{sec:intro}
Consistency is a desirable property of any statistical estimator, which informally means that in the limit of infinite data, the estimator converges to the correct quantity.
In the context of machine learning algorithms based on surrogate loss minimization, we usually use the notion of Fisher consistency, which means that the exact minimization of the expected surrogate loss leads to the exact minimization of the actual task loss.
It can be shown that Fisher consistency is closely related to the question of infinite-sample consistency (a.k.a. classification calibration)  of the surrogate loss with respect to the task loss (see \cite{bartlett06convexity,ramaswamy16calibrDim} for a detailed review).

The property of infinite-sample consistency (which we will refer to as simply consistency) shows that the minimization of a particular surrogate is the right problem to solve, but it becomes especially attractive when one can actually minimize the surrogate, which is the case, e.g, when the surrogate is convex.
Consistency of convex surrogates has been the central question of many studies for such problems as binary classification~\cite{bartlett06convexity,zhang2004annals,steinwart07}, multi-class classification~\cite{zhang04,tewari07,pires2013riskbounds,ramaswamy16calibrDim}, ranking~\cite{duchi10,buffoni2011learning,calauzenes12,ramaswamy13rankSurrogates,pedregosa15ordivalreg} and, more recently, structured prediction~\cite{ciliberto16,osokin17consistency}.

Recently, \citet{osokin17consistency} have pinpointed that in some cases minimizing a consistent convex surrogate might be not sufficient for efficient learning.
In particular, when the number of possible predictions is large (which is typically the case in the settings of structured prediction and ranking) reaching adequately small value of the expected task loss can be practically impossible, because one would need to optimize the surrogate to high accuracy, which requires an intractable number of iterations of the optimization algorithm.

It also turns out~\cite{osokin17consistency} that the possibility of efficient learning is related to the structure of the task loss.
The 0-1 loss, which does not make distinction between different kinds of errors, shows the worst case behavior.
However, more structured losses, e.g., the Hamming distance between sequence labelings, allow efficient learning if the score vector is designed appropriately (for the Hamming distance, the score for a complete configuration should be decomposable into the sum of scores for individual elements).

However, the analysis of~\citet{osokin17consistency} gives non-trivial conclusions only for consistent surrogates.
At the same time it is known that inconsistent surrogates often work well in practice (for example, the Crammer-Singer formulation of multi-class SVM~\cite{crammer01}, or its generalization structured SVM~\cite{taskar03,tsochantaridis05}).
There have indeed been several works to analyze inconsistent surrogates~\cite{long2013consistency,ramaswamy13rankSurrogates,calauzenes12,osokin17consistency}, but they usually end the story with proving that some surrogate (or a family or surrogates) is not consistent.

\textbf{Contributions.} In this work, we look at the problem from a more quantitative angle and analyze to which extent inconsistent surrogates can be useful for learning.
We focus on the same setting as~\cite{osokin17consistency} and generalize their results to the case of inconsistent surrogates (their bounds are trivial for these cases) to be able to draw non-trivial conclusions.
The main technical contribution consists in a tighter lower bound on the calibration function (Theorem~\ref{th:lowerBoundCalibrationFunctionInConsistent}), which is strictly more general than the bound of~\cite{osokin17consistency}.
Notably, our bound is non-trivial in the case when the surrogate is not consistent and quantifies to which degree learning with inconsistent surrogates is possible.
We further study the behavior of our bound in two practical scenarios: multi-class classification with a tree-structured loss and ranking with the mean average precision (mAP) loss.
For the tree-structured loss, our bound shows that there can be a trade-off between the best achievable accuracy and the speed of convergence.
For the mAP loss, we use our tools to study the (non-)existence of consistent convex surrogates of a particular dimension (an important issue for the task of ranking~\cite{duchi10,buffoni2011learning,calauzenes12,ramaswamy13rankSurrogates,ramaswamy16calibrDim}) and quantify to which extent our quadratic surrogate with the score vector of insufficient dimension is consistent.

This paper is organized as follows.
First, we introduce the setting we work with in Section~\ref{sec:notation} and review the key results of~\cite{osokin17consistency} in Section~\ref{sec:notation:calibration}.
In Section~\ref{sec:lowerBound}, we prove our main theoretical result, which is a new lower bound on the calibration function.
In Section~\ref{sec:losses}, we analyze the behavior of our bound for the two different settings: multi-class classification and ranking (the mean average precision loss).
Finally, we review the related works and conclude in Section~\ref{sec:conclusion}.

\section{Notation and Preliminaries} \label{sec:notation}
In this section, we introduce our setting, which closely follows~\cite{osokin17consistency}.
We denote the input features by $\inputvarv \in \inputdomain$ where~$\inputdomain$ is the input domain.
The particular structure of~$\inputdomain$ is not of the key importance for this study.
The output variables, that are in the center of our analysis, will be denoted by~$\outputvarv \in \outputdomain$ with~$\outputdomain$ being the set of possible predictions or the output domain.\footnote{The output domain~$\outputdomain$ itself can depend on the vector of input features~$\inputvarv$ (for example, if $\inputvarv$ can represent sequences of different lengths and the length of the output sequence has to equal the length of the input), but we will not use this dependency and omit it for brevity.}
In such settings as structured prediction or ranking, the predictions are very high-dimensional and with some structure that is useful to model explicitly (for example, a sequence, permutation or image).

The central object of our study is the \emph{loss function} $\lossmatrix(\outputvarv,\gtvarv) \geq 0$ that represents the cost of making the prediction~$\outputvarv \in \outputdomain$ when the ground-truth label is~$\gtvarv \in \gtdomain$.
Note that in some applications of interest the sets~$\outputdomain$ and~$\gtdomain$ are different.
For example, in ranking with the mean average precision (mAP) loss function (see Section~\ref{sec:rankingLosses} and, e.g., \citep{ramaswamy13rankSurrogates} for the details), the set~$\outputdomain$ consists of all the permutations of the items (to represent the ranking itself), but the set~$\gtdomain$ consists of all the subsets of items (to represent the set of relevant items, which is the ground-truth annotation in this setting).
In this paper, we only study the case when both~$\outputdomain$ and~$\gtdomain$ are finite.
We denote the cardinality of $\outputdomain$ by $\outputvarcard$, and the cardinality of $\gtdomain$ by $\gtvarcard$.
In this case, the loss function can be encoded as a matrix~$\lossmatrix$ of size $\outputvarcard \times \gtvarcard$.

In many applications of interest, both quantities~$\outputvarcard$ and~$\gtvarcard$ are exponentially large in the size of the natural dimension of the input~$\inputvarv$.
For example, in the task of sequence labeling, both~$\outputvarcard$ and~$\gtvarcard$ are equal to the number of all possible sequences of symbols from a finite alphabet.
In the task of ranking (the mAP formulation), $\outputvarcard$ is equal to the number of permutations of items and $\gtvarcard$ is equal to the number of item subsets.

Following usual practices, we work with the prediction model defined by a (learned) vector-valued \emph{score function} $\scorefunc: \inputdomain \to \R^{\outputvarcard}$, which defines a scalar score $\scorefunc_\outputvarv(\inputvarv)$ for each possible output~$\outputvarv \in \outputdomain$.
The final prediction is then chosen as an output configuration with the maximal score:
\begin{equation}
\label{eq:predictor}
\predictor(\scorefunc(\inputvarv)) := \argmax_{\outputvarv \in \outputdomain} \scorefunc_{\outputvarv}(\inputvarv).
\end{equation}
If the maximal score is given by multiple outputs~$\outputvarv$ (so-called \emph{ties}), the predictor follows a simple deterministic tie-breaking rule and picks the output appearing first in some predefined ordering on $\outputdomain$.

In this setup, learning consists in finding a \emph{score function}~$\scorefunc$ for which the predictor gives the smallest expected loss with features~$\inputvarv$ and labels~$\gtvarv$ coming from an unknown data-generating distribution~$\data$:
\begin{equation}
\label{eq:risk}
\risk_\lossmatrix(\scorefunc) := \E_{(\inputvarv, \gtvarv) \sim \data} \;\lossmatrix\bigl( \predictor(\scorefunc(\inputvarv)), \gtvarv \bigr).
\end{equation}
The quantity~$\risk_\lossmatrix(\scorefunc)$ is usually referred to as the actual (or population) \emph{risk} based on the loss~$\lossmatrix$.
Minimizing the actual risk directly is usually difficult (because of non-convexity and non-continuity of the predictor~\eqref{eq:predictor}).
The standard approach is to substitute~\eqref{eq:risk} with another objective, a \emph{surrogate risk} (or the $\surrogateloss$-risk), which is easier for optimization (in this paper, we only consider convex surrogates):
\begin{equation}
\label{eq:surrogateRisk}
\risk_\surrogateloss(\scorefunc) := \E_{(\inputvarv, \gtvarv) \sim \data} \;\surrogateloss( \scorefunc(\inputvarv), \gtvarv ),
\end{equation}
where we will refer to the function $\surrogateloss: \R^\outputvarcard \times \gtdomain \to \R$  as the~\emph{surrogate loss}.
To make the minimization of~\eqref{eq:surrogateRisk} well-defined, we will always assume the surrogate loss~$\surrogateloss$ to be bounded from below and continuous.

The surrogate loss should be chosen in such a way that the minimization of~\eqref{eq:surrogateRisk} also leads to the minimization of~\eqref{eq:risk}, i.e., to the solution of the original problem.
The property of consistency of the surrogate loss is an approach to formalize this intuition, i.e., to guarantee that no matter the data-generating distribution, minimizing~\eqref{eq:surrogateRisk} w.r.t.~$\scorefunc$ implies minimizing~\eqref{eq:risk} w.r.t.~$\scorefunc$ as well (both of these are possible only in the limit of infinite data and computational budget).
\citet{osokin17consistency} quantified what happens if the surrogate risk is minimized approximately by translating the optimization error of~\eqref{eq:surrogateRisk} to the optimization error of~\eqref{eq:risk}.
The main goal of this paper is to generalize this analysis to the cases when the surrogate is not consistent and to show that there can be trade-offs between the minimum value of the actual risk that can be achieved by minimizing an inconsistent surrogate and the speed with which this minimum can be achieved.

\section{Calibration Functions and Consistency} \label{sec:notation:calibration}
In this section, we review the approach of~\citet{osokin17consistency} for studying consistency in the context of structured prediction.
The first part of the analysis establishes the connection between the minimization of the actual risk~$\risk_\lossmatrix$~\eqref{eq:risk} and the surrogate risk $\risk_\surrogateloss$~\eqref{eq:surrogateRisk} via the so-called \emph{calibration function} (see Definition~\ref{def:calibrationFunc}~\cite[and references therein]{osokin17consistency}).
This step is usually called \emph{non-parametric} (or pointwise) because it does not explicitly model the dependency of the scores~$\scorev := \scorefunc(\inputvarv)$ on the input variables~$\inputvarv$.
The second part of the analysis establishes the connection with an optimization algorithm allowing to make a statement about how many iterations would be enough to find a predictor that is (in expectation) within $\eps$ of the global minimum of the actual risk~$\risk_\lossmatrix$.

\textbf{Non-parametric analysis.}
The standard non-parametric setting considers all measurable score functions~$\scorefunc$ to effectively ignore the dependency on the features~$\inputvarv$.
As noted by~\cite{osokin17consistency}, it is beneficial to consider a restricted set of the score functions~$\scorefuncset_\scoresubset$ that consists of all vector-valued Borel measurable functions $\scorefunc: \inputdomain \to \scoresubset$ where $\scoresubset \subseteq \R^\outputvarcard$ is a subspace of allowed score vectors.
Compatibility of the subspace~$\scoresubset$ and the loss function~$\lossmatrix$ will be a crucial point of this paper.
Note that the analysis is still non-parametric because the dependence on~$\inputvarv$ is not explicitly modeled.

Within the analysis, we will use the \emph{conditional} actual and surrogate risks defined as the expectations of the corresponding losses w.r.t.\ a categorical distribution~$\qv$ on the set of annotations~$\gtdomain$, $\gtvarcard := |\gtdomain|$:
\begin{equation}
\label{eq:conditionalRisks}
\lossweighted(\scorev, \qv)
:=
\sum\nolimits_{\gtvarv=1}^\gtvarcard q_\gtvarv \lossmatrix( \predictor(\scorev), \gtvarv ),
\quad
\surrogateweighted(\scorev, \qv)
:=
\sum\nolimits_{\gtvarv=1}^\gtvarcard q_\gtvarv \surrogateloss( \scorev, \gtvarv ).
\end{equation}
Hereinafter, we represent an $\gtvarcard$-dimensional  categorical distribution $\qv$ as a point in the probability simplex~$\simplex_\gtvarcard$ and use the symbol~$q_\gtvarv$ to denote the probability of the $\gtvarv$-th outcome.
Using this notation, we can rewrite the risk~$\risk_\lossmatrix$ and surrogate risk $\risk_\surrogateloss$ as
\begin{equation}
\label{eq:risks}
\risk_\lossmatrix(\scorefunc)
=
\E_{\inputvarv \sim \data_{\inputdomain}} \; \lossweighted(\scorefunc(\inputvarv), \P_\data(\cdot \mid \inputvarv)),
\quad
\risk_\surrogateloss(\scorefunc)
=
\E_{\inputvarv \sim \data_{\inputdomain}}\;\surrogateweighted(\scorefunc(\inputvarv), \P_\data(\cdot \mid \inputvarv)),
\end{equation}
where $\data_{\inputdomain}$ is the marginal distribution of $\inputvarv$ and $\P_\data(\cdot \:| \: \inputvarv)$ denotes the conditional distribution of $\gtvarv$ given $\inputvarv$ (both defined for the joint data-generating distribution $\data$).

For each score vector~$\scorev \in \scoresubset$ and a distribution~$\qv \in \simplex_\gtvarcard$ over ground-truth labels, we now define the \emph{excess} actual and surrogate risks
\begin{equation}
\label{eq:excessRisks}
\excess\surrogateweighted(\scorev, \qv)
=
\surrogateweighted(\scorev, \qv) - \inf_{\hat{\scorev} \in \scoresubset} \surrogateweighted(\hat{\scorev}, \qv), \quad
\excess\lossweighted(\scorev, \qv)
=
\lossweighted(\scorev, \qv) - \inf_{\hat{\scorev} \in \R^{\outputvarcard}} \lossweighted(\hat{\scorev}, \qv),
\end{equation}
which show how close the current conditional actual and surrogate risks are to the corresponding minimal achievable conditional risks (depending only on the distribution~$\qv$).
Note that the two infima in~\eqref{eq:excessRisks} are defined w.r.t.\ different sets of score vectors.
For the surrogate risk, the infimum is taken w.r.t.\ the set of allowed scores~$\scoresubset$ capturing only the scores obtainable by the learning process.
For the actual risk, the infimum is taken w.r.t.\ the set of all possible scores~$\R^\outputvarcard$ including score vectors that cannot be learned.
This distinction is important when analyzing inconsistent surrogates and allows to characterize the \emph{approximation error} of the selected function class.\footnote{Note that~\citet{osokin17consistency} define the excess risks by taking both infima w.r.t.\ the the set of allowed scores~$\scoresubset$, which is subtly different from us. The results of the two setups are equivalent in the cases of consistent surrogates, which are the main focus of~\citet{osokin17consistency}, but can be different in inconsistent cases.}

We are now ready to define the \emph{calibration function}, which is the final object of the non-parametric part of the analysis.
Calibration functions directly show how well one needs to minimize the surrogate risk to guarantee that the excess of the actual risk is smaller than~$\eps$.
\begin{definition}[Calibration function, \cite{osokin17consistency}]
    \label{def:calibrationFunc}
    For a task loss~$\lossmatrix$, a surrogate loss~$\surrogateloss$, a set of feasible scores~$\scoresubset$, the \emph{calibration function}~$\calibrationfunc_{\surrogateloss,\lossmatrix,\scoresubset}(\eps)$ is defined as:
    \begin{align}
    \label{eq:calibrationfunc}
    \calibrationfunc_{\surrogateloss,\lossmatrix,\scoresubset}(\eps)
    :=
    \;\;\;\;
    \inf_{\mathclap{\scorev \in \scoresubset, \;\qv \in \simplex_\gtvarcard}} \;&\;\;\;\;\excess\surrogateweighted(\scorev, \qv) \\
    \label{eq:calibrationfunc:epsConstr}
    \text{\textup{s.t.}} \:&\;\;\;\;\excess\lossweighted(\scorev, \qv) \geq \eps,
    \end{align}
    where $\eps \geq 0$ is the target accuracy.
    We set $\calibrationfunc_{\surrogateloss,\lossmatrix,\scoresubset}(\eps)$ to $+\infty$ when the feasible set is empty.
\end{definition}
By construction, $\calibrationfunc_{\surrogateloss,\lossmatrix,\scoresubset}$ is non-decreasing on $[0, +\infty)$, $\calibrationfunc_{\surrogateloss,\lossmatrix,\scoresubset}(\eps) \geq 0$ and $\calibrationfunc_{\surrogateloss,\lossmatrix,\scoresubset}(0) = 0$.
The calibration function also provides the so-called \emph{excess risk bound}
\begin{equation}
\label{eq:excessRiskBound}
\calibrationfunc_{\surrogateloss,\lossmatrix,\scoresubset}(\excess\lossweighted(\scorev, \qv)) \leq \excess\surrogateweighted(\scorev, \qv), \; \forall \scorev \in \scoresubset, \; \forall \qv \in \simplex_\gtvarcard,
\end{equation}
which implies the formal connection between the surrogate and task risks~\cite[Theorem 2]{osokin17consistency}.

The calibration function can fully characterize consistency of the setting defined by the surrogate loss, the subspace of scores and the task loss.
The maximal value of $\eps$ at which the calibration function $\calibrationfunc_{\surrogateloss,\lossmatrix,\scoresubset}(\eps)$ equals zero shows the best accuracy on the actual loss that can be obtained~\cite[Theorem 6]{osokin17consistency}.
The notion of level-$\consbreakpoint$ consistency captures this effect.
\begin{definition}[level-$\consbreakpoint$ consistency, \cite{osokin17consistency}]
    \label{def:consistency}
    A surrogate loss~$\surrogateloss$ is \emph{consistent up to level~$\consbreakpoint \geq 0$} w.r.t.\ a task loss~$\lossmatrix$ and a set of scores~$\scoresubset$ if and only if the calibration function satisfies $\calibrationfunc_{\surrogateloss,\lossmatrix,\scoresubset}(\eps) > 0$ for all $\eps > \consbreakpoint$
    and there exists $\hat\eps > \consbreakpoint$ such that $\calibrationfunc_{\surrogateloss,\lossmatrix,\scoresubset}(\hat\eps)$ is finite.
\end{definition}
The case of level-$0$ consistency corresponds to the classical consistent surrogate and Fisher consistency.
When $\consbreakpoint > 0$, the surrogate is not consistent, meaning that the actual risk cannot be minimized globally.
However, \citet[Appendix E.4]{osokin17consistency} give an example where even though constructing a consistent setting is possible (by the choice of the score subspace~$\scoresubset$), it might still be beneficial to use only a level-$\consbreakpoint$ consistent setting because of the exponentially faster growth of the calibration function.
The main contribution of this paper is a lower bound on the calibration function (Theorem~\ref{th:lowerBoundCalibrationFunctionInConsistent}), which is non-zero for $\consbreakpoint > 0$ and thus can be used to obtain convergence rates in inconsistent settings.

\textbf{Optimization and learning guarantees; normalizing the calibration function.}
\citet{osokin17consistency} note that the scale of the calibration function is not defined, i.e., if one multiplies the surrogate loss by some positive constant, the calibration function is multiplied by the same constant as well.
One way to define a ``natural normalization'' is to use a scale-invariant convergence rate of a stochastic optimization algorithm.
\citet[Section 3.3]{osokin17consistency} applied the classical online ASGD~\cite{nemirovski09} (under the well-specification assumption) and got the sample complexity (and the convergence rate of ASGD at the same time) result saying that~$\maxiter^*$ steps of ASGD are sufficient to get~$\eps$-accuracy on the \emph{task loss} (in expectation), where $\maxiter^*$ is computed as follows:
\begin{equation}
\label{eq:calFuncNormalized}
\maxiter^* := \frac{4 D^2 M^2}{\vphantom{(\big(\bigr)}\check{\calibrationfunc}_{\surrogateloss,\lossmatrix,\scoresubset}^2(\eps) }.
\end{equation}
Here the quantity~$\maxiter^*$ depends on a convex lower bound~$\check{\calibrationfunc}_{\surrogateloss,\lossmatrix,\scoresubset}(\eps)$ on the calibration function~$\calibrationfunc_{\surrogateloss,\lossmatrix,\scoresubset}(\eps)$ and the constants $D$, $M$, which appear in the convergence rate of ASGD: $D$ is an upper bound on the norm of an optimal solution and $M^2$ is an upper bound on the expected square norm of the stochastic gradient.
\citet{osokin17consistency} show how to bound the constant~$DM$ for a very specific quadratic surrogate defined below (see Section~\ref{sec:notation:boundsQuadSurrogate}).

\subsection{Bounds for the Quadratic Surrogate} \label{sec:notation:boundsQuadSurrogate}
The major complication in applying and interpreting the theoretical results presented in Section~\ref{sec:notation:calibration} is the complexity of computing the calibration function.
\citet{osokin17consistency} analyzed the calibration function only for the quadratic surrogate
\begin{equation}
\label{eq:quadrLoss}
\surrogatelossquad( \scorev, \outputvarv )
:=
\frac{1}{2\outputvarcard} \| \scorev + \lossmatrix(:, \gtvarv)\|_2^2 = \frac{1}{2\outputvarcard} \sum\nolimits_{\outputvarv \in \outputdomain} (\score_\outputvarv^2 + 2\score_\outputvarv L(\outputvarv,\gtvarv)  + L(\outputvarv,\gtvarv)^2).
\end{equation}
For any task loss~$\lossmatrix$, this surrogate is consistent whenever the subspace of allowed scores is rich enough, i.e., the subspace of scores~$\scoresubset$ fully contains~$\colspace(\lossmatrix)$.
To connect with optimization, we assume a parametrization of the subspace~$\scoresubset$ as a span of the columns of some matrix~$\scorematrix$, i.e., $\scoresubset = \colspace(\scorematrix) = \{\scorev = \scorematrix\scoreparamv \mid \scoreparamv\in \R^\scoresubspacedim\}$.\footnote{We do a pointwise analysis in this section, so we are not modeling the dependence of~$\scoreparamv$ on the features~$\inputvarv$.
However, in an actual implementation, the vector~$\scoreparamv$ should be a function of the features~$\inputvarv$ coming from some flexible family such as a RKHS or some neural networks.}
In the interesting settings, the dimension~$\scoresubspacedim$ is much smaller than both~$\outputvarcard$ and~$\gtvarcard$.
Note that to compute the gradient of the objective~\eqref{eq:quadrLoss} w.r.t.\ the parameters~$\scoreparamv$, one needs to compute matrix products $\scorematrix^\transpose \scorematrix \in \R^{\scoresubspacedim \times \scoresubspacedim}$ and $\scorematrix^\transpose \lossmatrix(:, \gtvarv) \in \R^\scoresubspacedim$, which are usually both of feasible sizes, but require exponentially big sum ($\outputvarcard$ summands) inside.
Computing these quantities can be seen as some form of inference required to run the learning process.

\citet{osokin17consistency} proved a lower bound on the calibration functions for the quadratic surrogate~\eqref{eq:quadrLoss}~\cite[Theorem 7]{osokin17consistency}, which we now present to contrast our result presented in Section~\ref{sec:lowerBound}.
When the subspace of scores~$\scoresubset$ contains~$\colspace(\lossmatrix)$, $\colspace(\lossmatrix) \subseteq \scoresubset$, implying that the setting is consistent, the calibration function is bounded from below by $\min_{i\neq j} \frac{\eps^2}{2\outputvarcard \|\proj_{\scoresubset} \Delta_{ij}\|_2^2}$, where $\proj_{\scoresubset}$ is the orthogonal projection on the subspace~$\scoresubset$ and $\Delta_{ij} := \unit_i - \unit_j \in \R^\outputvarcard$ with $\unit_c$ being the $c$-th basis vector of the standard basis in~$\R^\outputvarcard$.
They also showed that for some very structured losses (Hamming and block 0-1 losses), the quantity $\outputvarcard \|\proj_{\scoresubset} \Delta_{ij}\|_2^2$ is not exponentially large and thus the calibration function suggests that efficient learning is possible.
One interesting case not studied by~\citet{osokin17consistency} is the situation where the subspace of scores~$\scoresubset$ does not fully contain the subspace~$\colspace(\lossmatrix)$.
In this case, the surrogate might not be consistent but still lead to effective and efficient practical algorithms. %

\textbf{Normalizing the calibration function.} 
The normalization constant~$DM$ appearing in~\eqref{eq:calFuncNormalized} can also be computed for the quadratic surrogate~\eqref{eq:quadrLoss} under the assumption of well-specification (see \cite[Appendix F]{osokin17consistency} for details).
In particular, we have
$
DM
=
\Lmax^2 \xi(\condnum(\scorematrix) \sqrt{\scoresubspacedim} \rkhsbound  \qbound), \; \xi(z) = z^2 + z,
$%
\;where~$\Lmax$ denotes the maximal value of all elements in~$\lossmatrix$,
$\condnum(\scorematrix)$ is the condition number of the matrix~$\scorematrix$ %
and $\scoresubspacedim$ in an upper bound on the rank of~$\scoresubset$.
The constants $\rkhsbound$ and $\qbound$ come from the kernel ASGD setup and, importantly, depend only on the data distribution, but not on the loss~$\lossmatrix$ or score matrix~$\scorematrix$.
Note that for a given subspace~$\scoresubset$, the choice of matrix~$\scorematrix$ is arbitrary and it can always be chosen as an orthonormal basis of~$\scoresubset$ giving a $\condnum(\scorematrix)$ of one.
However, such~$\scorematrix$ can lead to inefficient prediction~\eqref{eq:predictor}, which makes the whole framework less appealing.
Another important observation coming from the value of~$DM$ is the justification of the~$\frac{1}{\outputvarcard}$ scaling in front of the surrogate~\eqref{eq:quadrLoss}.

\section{Calibration Function for Inconsistent Surrogates} \label{sec:lowerBound}
Our main result generalizes the Theorem 7 of~\cite{osokin17consistency} to the case of inconsistent surrogates (the key difference consists in the absence of the assumption $\colspace(\lossmatrix) \subseteq \scoresubset$).

\begin{theorem}[Lower bound on the calibration function $\calibrationfunc_{\surrogatelossquad,\lossmatrix,\scoresubset}(\eps)$]
    \label{th:lowerBoundCalibrationFunctionInConsistent}
    For any task loss~$\lossmatrix$, its quadratic surrogate~$\surrogatelossquad$, and a score subspace~$\scoresubset$, the calibration function is bounded from below:
    \begin{equation}
    \label{eq:lowerBoundcalibrationFunctionInconsistent}
    \calibrationfunc_{\surrogatelossquad,\lossmatrix,\scoresubset}(\eps)
    \geq
    \min_{i\neq j} 
    \max_{v \geq 0} \frac{(\eps v - \xi_{ij}(v))_+^2}{2\outputvarcard \|\proj_{\scoresubset} \Delta_{ij}\|_2^2},\quad 
    \text{where\;\;\;} \xi_{ij}(v) := \bigl\| \:\lossmatrix^\transpose (v \id_\outputvarcard - \proj_{\scoresubset}) \Delta_{ij} \: \bigr\|_\infty,
    \end{equation}
    where $\proj_{\scoresubset}$ is the orthogonal projection on the subspace~$\scoresubset$, $(x)_+^2 := [x > 0] x^2$ is the truncation of the parabola to its right branch and $\Delta_{ij} := \unit_i - \unit_j \in \R^k$ with $\unit_c\in \R^k$ being the $c$-th column of the identity matrix $\id_\outputvarcard$.
    By convention, if both numerator and denominator of~\eqref{eq:lowerBoundcalibrationFunctionInconsistent} equal zero the whole bound equals zero.
    If only the denominator equals zero then the whole bound equals infinity (the particular pair of~$i$ and~$j$ is effectively not considered).
\end{theorem}
The proof of Theorem~\ref{th:lowerBoundCalibrationFunctionInConsistent} starts with using the idea of~\cite{osokin17consistency} to compute the calibration function by solving a collection of convex quadratic programs (QPs).
Then we diverge from the proof of~\cite{osokin17consistency} (because it leads to a non-informative bound in inconsistent settings).
For each of the formulated QPs, we construct a dual by using the approach of~\citet{dorn60duality}.
The dual of Dorn is convenient for our needs because it does not require inverting the matrix defining the quadratic terms (compared to the standard Lagrangian dual).
The complete proof is given in Appendix~\ref{sec:LowerBoundThmProof}.

\begin{remark}
    The numerator of the bound~\eqref{eq:lowerBoundcalibrationFunctionInconsistent} explicitly specifies the point at which the bound becomes non-zero, implying level-$\consbreakpoint$ consistency with~$\consbreakpoint = \frac{\xi_{ij}(v)}{v}$ for the values of~$i$, $j$, $v$ that are active for a particular $\eps$.
    The quantity~$\frac{v^2}{2\outputvarcard \|\proj_{\scoresubset} \Delta_{ij}\|_2^2}$ bounds the weight of the $\eps^2$ term in the calibration function after it leaves zero.
    Moving the quantity~$v$ defines the trade-off between the slope, which is related to the convergence speed of the algorithm, and the value of~$\consbreakpoint$ defining the best achievable accuracy.
\end{remark}

\begin{remark}
    If we have conditions of Theorem 7 of~\cite{osokin17consistency} satisfied, i.e., $\colspace(\lossmatrix) \subseteq \scoresubset$, then the vector~$\lossmatrix^\transpose (\id_\outputvarcard - \proj_{\scoresubset}) \Delta_{ij}$ equals zero and $\xi_{ij}(v)$ becomes $ |v - 1|\: \| \lossmatrix^\transpose \Delta_{ij} \|_\infty$, which equals zero when $v = 1$.
    It might seem that having $v > 1$ can potentially give us a tighter lower bound than Theorem~7~\cite{osokin17consistency} even in consistent cases.
    However, the quantity~$\| \lossmatrix^\transpose \Delta_{ij} \|_\infty$ upper bounds the maximal possible (w.r.t.\ the conditional distribution  $\P_\data(\cdot \:| \: \inputvarv)$) value of the excess task loss for a fixed pair~$i$, $j$ leading to the identity $ v \eps  - |v - 1|\: \| \lossmatrix^\transpose \Delta_{ij} \|_\infty = \| \lossmatrix^\transpose \Delta_{ij} \|_\infty$ for $\eps = \| \lossmatrix^\transpose \Delta_{ij} \|_\infty$ and $v \geq 1$.
    Together with the convexity of the function~$(x)^2_+$, this implies that the best possible value of $v$ in consistent settings equals one.
\end{remark}

\begin{remark}
    Setting~$v$ in~\eqref{eq:lowerBoundcalibrationFunctionInconsistent} to any non-negative constant gives a valid lower bound.
    In particular, setting~$v$ to 1 (while potentially making the bound less tight) highlights the separation between the weight of the quadratic term and the best achievable accuracy~$\consbreakpoint$.
    The bound now reads as follows:\\[-2mm]
    \begin{equation}
    \label{eq:lowerBoundcalibrationFunctionInconsistentNoV}
    \calibrationfunc_{\surrogatelossquad,\lossmatrix,\scoresubset}(\eps)
    \geq
    \min_{i\neq j} 
    \frac{(\eps - \xi_{ij})_+^2}{2\outputvarcard \|\proj_{\scoresubset} \Delta_{ij}\|_2^2},\quad 
    \text{where\;\;\;} \xi_{ij} := \bigl\|\:\lossmatrix^\transpose (\id_\outputvarcard - \proj_{\scoresubset}) \Delta_{ij} \:\bigr\|_\infty.
    \end{equation}
    Note that the weight of the $\eps^2$ term now equals the corresponding coefficient of the bound of Theorem~7~\cite{osokin17consistency}.
    Notably, this weight depends only on the score subspace~$\scoresubset$, but not on the loss~$\lossmatrix$.
\end{remark}

\vspace{-2mm}\section{Bounds for Particular Losses} \label{sec:losses}

\vspace{-2mm}\subsection{Multi-Class Classification with the Tree-Structured Loss} \label{sec:hierarhicalLoss}

\vspace{-2mm}As an illustration of the obtained lower bound~\eqref{eq:lowerBoundcalibrationFunctionInconsistent}, we consider the task of multi-class classification and the \emph{tree-structured loss}, which is defined for a weighted tree built on labels (such trees on labels often appear in settings with large number of labels, e.g., extreme classification~\cite{choromanska2013extreme}).
Leaves in the tree correspond to the class labels $\outputvar \in \outputdomain = \gtdomain$ and the loss function is defined as the length of the path~$\treedist$ between the leaves, i.e., $\lossmatrix_{\text{\textup{tree}}}(\gtvar , \outputvar) := \treedist(\gtvar, \outputvar)$.
To compute the lower bound exactly, we assume that the number of children $\nchildren_s$ and the weights of the edges connecting a node with its children $\frac{\treeweight_s}{2}$ are equal for all the nodes of the same depth level $s=0, \dots, \treedepth-1$ (see Figure~\ref{fig:treeillustration} in Appendix~\ref{sec:treeProofs} for an example of such a tree) and that $\sum_{s=0}^{\treedepth - 1} \alpha_s = 1$, which normalizes~$\Lmax$ to one.
To define the score matrix $\scoresubset_{\text{\textup{tree}}, \consistencydepth}$, we set the \emph{consistency depth} $\consistencydepth \in \{ 1, \dots,  \treedepth \}$ and restrict the scores $\scorev$ to be equal for the groups (blocks) of leaves that have the same ancestor on the level~$\consistencydepth$.
Let $B(i)$ be the set of leaves that have the same ancestor as a leaf $i$ at the depth $\consistencydepth$.
With this notation, we have $\scoresubset_{\text{\textup{tree}}, \consistencydepth} = \colspace{ \{ \sum_{i \in B(j)} \unit_i \mid j = 1, \dots, \outputvarcard \} }$.
Theorem~\ref{th:lowerBoundCalibrationFunctionInConsistent} gives us the bound (see Appendix~\ref{sec:treeProofs}):\\[-2mm]
\begin{equation}
\label{eq:consistencypaperboundgeneralized}
\calibrationfunc_{\surrogatelossquad, \lossmatrix_{\text{\textup{tree}}},\scoresubset_{\text{\textup{tree}}, \consistencydepth}}(\eps) \geq
\lbrack \epsilon > \blockdist_\consistencydepth \rbrack
\frac{(\blockdist_\consistencydepth - \avgblockdist_\consistencydepth + \alpha_{\consistencydepth -1})^2}{(\frac{\blockdist_\consistencydepth}{2} + \alpha_{\consistencydepth - 1})^2}\frac{(\epsilon - \frac{\blockdist_\consistencydepth}{2})_+^2}{4\numblocks_\consistencydepth} ,
\end{equation}
where $\numblocks_\consistencydepth$, $\avgblockdist_\consistencydepth \!:=\! \frac{1}{|B(j)|} \sum_{i \in B(j)} \treedist(i, j) \!=\! \sum_{s=\consistencydepth}^{\treedepth - 1} \alpha_s \frac{ (\prod\nolimits_{s'=\consistencydepth}^{s} \nchildren_{s'} ) - 1}{\prod\nolimits_{s'=\consistencydepth}^{s} \nchildren_{s'}}$  and $\blockdist_\consistencydepth \!:=\! \max_{i \in B(j)} \treedist(i, j) \!=\! \sum_{s = \consistencydepth}^{\treedepth - 1} \treeweight_s$ are the number of blocks, the average and maximal distance within a block, respectively.%

Now we discuss the behavior of the bound~\eqref{eq:consistencypaperboundgeneralized} when changing the truncation level~$\consistencydepth$.
With the growth of~$\consistencydepth$, the level of consistency~$\blockdist_\consistencydepth$ goes to~$0$ indicating that more labels can be distinguished.
At the same time, we have $\tfrac{\blockdist_\consistencydepth}{2} \leq \avgblockdist_\consistencydepth$ for the trees we consider and thus the coefficient in front of the $\eps^2$ term can be bounded from above by $\frac{1}{4 \numblocks_\consistencydepth}$, which means that the lower bound on the calibration function decreases at an exponential rate with the growth of $\consistencydepth$. These arguments show the trade-off between the level of consistency and the coefficient of~$\eps^2$ in the calibration function.

Finally, note that the mixture of 0-1 and block 0-1 losses considered in~\cite[Appendix E.4]{osokin17consistency} is an instance of the tree-structured loss with $\treedepth = 2$.
Their bound~\cite[Proposition 17]{osokin17consistency} matches~\eqref{eq:consistencypaperboundgeneralized} up to the difference in the definition of the calibration function (they do not have the $\lbrack \epsilon > \blockdist_\consistencydepth \rbrack$ multiplier because they do not consider pairs of labels that fall in the same block).

\vspace{-1mm}\subsection{Mean Average Precision (mAP) Loss for Ranking} \label{sec:rankingLosses}
\vspace{-2mm}The mAP loss, which is a popular way of measuring the quality of ranking, has attracted significant attention from the consistency point of view~\cite{buffoni2011learning,calauzenes12,ramaswamy13rankSurrogates}.
In the mAP setting, the ground-truth labels are binary vectors $\gtvarv \in \gtdomain = \{0, 1\}^\size$ that indicate the items relevant for the query (a subset of~$\size$ items-to-rank) and the prediction consists in producing a permutation of items $\routputvar \in \outputdomain$, $\outputdomain = \mathcal{S}_\size$.
The mAP loss is based on averaging the precision at different levels of recall and is defined as follows:\\[-2mm]
\begin{equation}
\label{eq:mapLoss}
\lossmatrix_{\text{\textup{mAP}}}(\routputvar, \gtvar)
:= 1 - \frac{1}{|\gtvarv|} \sum_{p: \gtvar_p=1}^{\size} \frac{1}{\sigma(p)} \sum_{q=1}^{\sigma(p)} \gtvar_{\routputvar^{-1}(q)}
= 1 - \sum_{p=1}^{\size} \sum_{q=1}^{p} \frac{1}{\max(\sigma(p), \sigma(q))} \frac{\gtvar_p \gtvar_q}{|\gtvarv|},
\end{equation}
where $\sigma(p)$ is the position of an item~$p$ in a permutation~$\sigma$, $\sigma^{-1}$ is the inverse permutation and $|\gtvarv| := \sum_{p=1}^{\size} \gtvar_p$.
The second identity provides a convenient form of writing the mAP loss~\cite{ramaswamy13rankSurrogates} showing that the loss matrix~$\lossmatrix_{\text{\textup{mAP}}}$ is of rank at most~$\frac12\size(\size+1)$.\footnote{\citet[Proposition 21]{ramaswamy16calibrDim} showed that the rank of~$\lossmatrix_{\text{\textup{mAP}}}$ is a least~$\frac12\size(\size+1) - 2$.}
The matrix~$\scorematrix_{\text{\textup{mAP}}} \in \R^{\size! \times \frac12\size(\size+1)}$ such that $(\scorematrix_{\text{\textup{mAP}}})_{\sigma, pq} := \frac{1}{\max(\sigma(p), \sigma(q))}$ is a natural candidate to define the score subspace~$\scoresubset$ to get the consistent setting with the quadratic surrogate~\eqref{eq:quadrLoss} (Eq.~\eqref{eq:mapLoss} implies that~$\colspace(\lossmatrix_{\text{\textup{mAP}}}) = \colspace( \scorematrix_{\text{\textup{mAP}}} )$).

However, as noted in Section 6 of~\cite{ramaswamy13rankSurrogates}, although the matrix $\scorematrix_{\text{\textup{mAP}}}$ is convenient from the consistency point of view (in the setup of~\cite{ramaswamy13rankSurrogates}), it leads to the prediction problem~$\max_{\sigma \in \mathcal{S}_\size} (\scorematrix_{\text{\textup{mAP}}} \scoreparamv )_\sigma$, which is a quadratic assignment problem (QAP), and most QAPs are NP-hard.

To be able to predict efficiently, it would be beneficial to have the matrix~$\scorematrix$ with $\size$ columns such that sorting the  $\size$-dimensional~$\scoreparamv$ would give the desired permutation.
It appears that it is possible to construct such a matrix by selecting a subset of columns of matrix~$\scorematrix_{\text{\textup{mAP}}}$.
We define~$\scorematrix_{\text{\textup{sort}}} \in \R^{\size! \times \size}$ by $(\scorematrix_{\text{\textup{sort}}})_{\sigma, p} := \frac{1}{\sigma(p)}$.
A solution of the prediction problem $\max_{\sigma \in \mathcal{S}_\size} (\scorematrix_{\text{\textup{sort}}} \scoreparamv )_\sigma$ is simply a permutation that sorts the elements of~$\scoreparamv \in \R^{\size}$ in the decreasing order (this statement follows from the fact that we can always increase the score~$(\scorematrix_{\text{\textup{sort}}} \scoreparamv )_\sigma = \sum_{p=1}^\size \frac{\scoreparam_p}{\sigma(p)}$ by swapping a pair of non-aligned items).

\begin{figure}[t]
    \begin{tabular}{c@{}c@{}c}
        \includegraphics[width=0.4\textwidth]{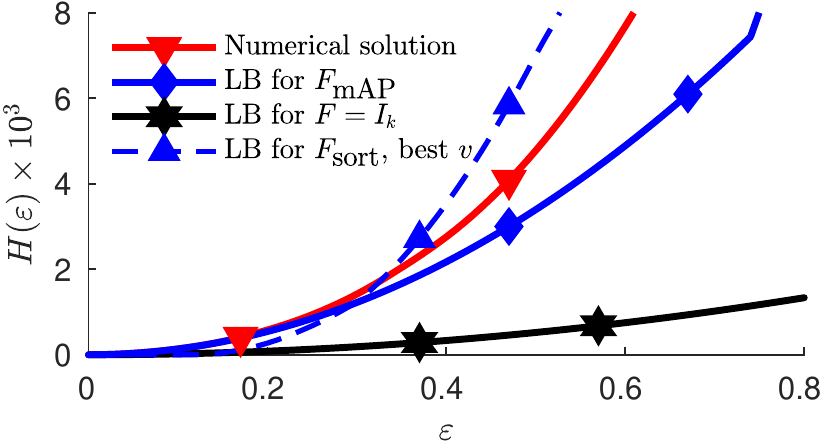}
        &
        \includegraphics[width=0.4\textwidth]{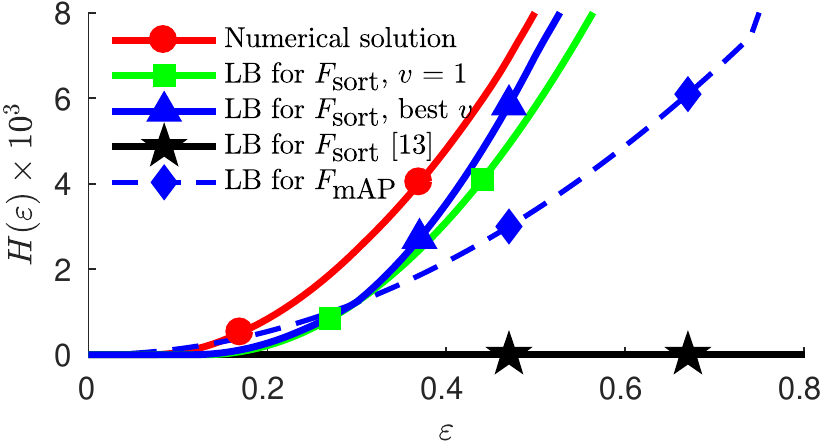}
        &
        \includegraphics[width=0.17\textwidth]{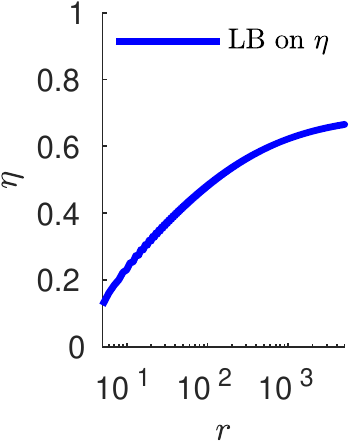}
        \\
        (a): Consistent surrogate with~$\scorematrix_{\text{\textup{mAP}}}$
        &
        (b): Inconsistent surrogate with~$\scorematrix_{\text{\textup{sort}}}$
        &
        (c): LB on~$\consbreakpoint$
    \end{tabular}
    \caption{%
        \textbf{Plot (a)} shows the calibration function~$\calibrationfunc_{\surrogatelossquad,\lossmatrix_{\text{\textup{mAP}}},\scoresubset_{\text{\textup{mAP}}}}(\eps)$ for $\lossmatrix_{\text{\textup{mAP}}}$ (red line) obtained numerically.
    The solid blue line~\cite[Theorem~7]{osokin17consistency} is its lower bound, LB, and the solid black line is the worst case bound obtained for~$\scorematrix = \id_{\size!}$ (which means not constructing an appropriate low-dimension~$\scoresubset$).
    Difference between the blue and the black lines is exponential (proportional to~$\size!$).
    The dashed blue line illustrates the inconsistent surrogate (note that it is zero for small~$\eps > 0$, but then grows faster than the solid blue line~-- the consistent setting).
    \textbf{Plot (b)} shows the calibration function~$\calibrationfunc_{\surrogatelossquad,\lossmatrix_{\text{\textup{mAP}}},\scoresubset_{\text{\textup{sort}}}}(\eps)$ (red line) obtained numerically (this setting is level-$\consbreakpoint$ consistent for $\consbreakpoint \approx 0.08$).
    The blue line~(Theorem~\ref{th:lowerBoundCalibrationFunctionInConsistent}) is its lower bound for the optimal value of~$v$ and the green line is the bound for~$v = 1$ (easier to obtain).
    The black line shows the zero-valued trivial bound from~\cite{osokin17consistency}.
    The dashed blue line shows $\calibrationfunc_{\surrogatelossquad,\lossmatrix_{\text{\textup{mAP}}},\scoresubset_{\text{\textup{mAP}}}}(\eps)$ for the consistent surrogate to compare the two settings.
    Note that in both plots (a) and (b) the solid blue lines are the lower bounds of the corresponding calibration functions (red lines), but the dashed blue lines are not (shown for comparison purposes).
    \textbf{Plot (c)} shows a lower bound on the point~$\eta$ where the exact calibration function~$\calibrationfunc_{\surrogatelossquad,\lossmatrix_{\text{\textup{mAP}}},\scoresubset_{\text{\textup{sort}}}}(\eps)$ stops being zero, indicating the level of consistency (Definition~\ref{def:consistency}).
    \label{fig:rankingPlots} \vspace{-4mm}}
\end{figure}

Most importantly for our study, the columns of the matrix $\scorematrix_{\text{\textup{sort}}}$ are a subset of the columns of the matrix~$\scorematrix_{\text{\textup{mAP}}}$, which indicates that learning with the convenient matrix~$\scorematrix_{\text{\textup{sort}}}$ might be sufficient for the mAP loss.
In what follows, we study the calibration functions for the loss matrix~$\lossmatrix_{\text{\textup{mAP}}}$ and score matrices $\scorematrix_{\text{\textup{mAP}}}$ and $\scorematrix_{\text{\textup{sort}}}$.
In Figure~\ref{fig:rankingPlots}a-b, we plot the calibration functions for both~$\scorematrix_{\text{\textup{mAP}}}$ and ~$\scorematrix_{\text{\textup{sort}}}$ and the lower bounds given by Theorem~\ref{th:lowerBoundCalibrationFunctionInConsistent}.
All the curves were obtained for~$\size = 5$ (computing the exact values of the calibration functions is exponential in $\size$).

Next, we study the behavior of the lower bound~\eqref{eq:lowerBoundcalibrationFunctionInconsistent} for large values of~$\size$.
In Lemma~\ref{th:rankingasymptotics} of Appendix~\ref{sec:rankingProofs}, we show that the denominator of the bound~\eqref{eq:lowerBoundcalibrationFunctionInconsistent} is not exponential in~$\size$ (we have $2\size! \|\proj_{\scoresubset_{\text{\textup{sort}}}} \Delta_{\pi \omega}\|_2^2 = O(\size)$).
We also know that~$\|\proj_{\scoresubset_{\text{\textup{sort}}}} \Delta_{\pi \omega}\|_2^2 \leq \|\proj_{\scoresubset_{\text{\textup{mAP}}}} \Delta_{\pi \omega}\|_2^2$ (because~$\scoresubset_{\text{\textup{sort}}}$ is a subspace of~$\scoresubset_{\text{\textup{mAP}}}$), which implies that the calibration function of the consistent setting grows not faster than the one of the inconsistent setting.
We can also numerically compute a lower bound on the point $\consbreakpoint$ until which the calibration function is guaranteed to be zero (for this we simply pick two permutations~$\pi$, $\omega$ and a labeling~$\gtvarv$ that delivers large values of $\bigr( \lossmatrix_{\text{\textup{mAP}}}^\transpose ( \id_\outputvarcard - \proj_{\scoresubset_{\text{\textup{sort}}}}) \Delta_{ij} \bigr)_{\gtvarv} \leq \xi_{\pi, \omega}(1)$).
Figure~\ref{fig:rankingPlots}c shows that the level of inconsistency~$\consbreakpoint$ grows with the growth of~$\size$, which makes the method less appealing for large-scale settings.

Finally, note that to run the ASGD algorithm for the quadratic surrogate~\eqref{eq:quadrLoss}, mAP loss and score matrix~$\scorematrix_{\text{\textup{sort}}}$, we need to efficiently compute $\scorematrix_{\text{\textup{sort}}}^\transpose \scorematrix_{\text{\textup{sort}}}$ and $\scorematrix_{\text{\textup{sort}}}^\transpose \lossmatrix_{\text{\textup{mAP}}}(:, \gtvarv)$. Lemmas~\ref{th:rankingprojection} and~\ref{th:rankingxijv} (see Appendix~\ref{sec:rankingProofs}) provide linear in~$\size$ time algorithms for doing this.
The condition number of~$\scoresubset_{\text{\textup{sort}}}$ grows as~$\Theta(\log \size)$ keeping the sample complexity bound~\eqref{eq:calFuncNormalized} well behaved.

\section{Discussion} \label{sec:conclusion}
\vspace{-2mm}\textbf{Related works.} Despite a large number of works studying consistency and calibration in the context of machine learning, there have been relatively few attempts to obtain guarantees for inconsistent surrogates.
The most popular approach is to study consistency under so-called \emph{low noise} conditions.
Such works show that under certain assumptions on the data generating distribution~$\data$ (usually these assumptions are on the conditional distribution of labels and are impossible to verify for real data) the surrogate of interest becomes consistent, whereas being inconsistent for general~$\data$.
\citet{duchi10} established such a result for the value-regularized linear surrogate for ranking (which resembles the pairwise disagreement, PD, loss).
\citet{ramaswamy13rankSurrogates} provided similar results for the mAP and PD losses for ranking and their quadratic surrogate.
Similarly to our conclusions, the mAP surrogate of~\cite{ramaswamy13rankSurrogates} is consistent with~$\frac{1}{2} \size(\size+1)$ parameters learned and only low-noise consistent with~$\size$ parameters learned.
\citet{long2013consistency} introduced a notion of realizable consistency w.r.t.\ a function class (they considered linear predictors), which is consistency w.r.t.\ the function class assuming the data distribution such that labels depend on features deterministically with this dependency being in the correct function class.
\citet{bendavid2012minimizing} worked in the agnostic setting for binary classification (no assumptions on the underlying~$\data$) and provided guarantees on the error of linear predictors when the margin was bounded by some constant (their work reduces to consistency in the limit case, but is more general).

\textbf{Conclusion.} Differently from the previous approaches, we do not put constraints on the data generating distribution, but instead study the connection between the surrogate and task losses by the means of the calibration function (following~\cite{osokin17consistency}), which represents the worst-case scenario.
For the quadratic surrogate~\eqref{eq:quadrLoss}, we can bound the calibration function from below in such a way that the bound is non-trivial in inconsistent settings (differently from~\cite{osokin17consistency}).
Our bound quantifies the level of inconsistency of a setting (defined by the used surrogate loss, task loss and parametrization of the scores) and allows to analyze when learning with inconsistent surrogates can be beneficial.
We illustrate the behavior of our bound for two tasks (multi-class classification and ranking) and show examples of conclusions that our approach can give.

\textbf{Future work.} It would be interesting to combine our quantitative analysis with the constraints on the data distribution, which might give adaptive calibration functions (in analogy to adaptive convergence rates in convex optimization: for example, SAGA~\cite{defazio14saga} has a linear convergence rate for strongly convex objectives and $1/t$ rate for non-strongly convex ones),
and with the recent results of~\citet{pillaud2018} showing that under some low-noise assumptions even slow convergence of the surrogate objective can imply exponentially fast convergence of the task loss.

\subsubsection*{Acknowledgements}
This work was partly supported by Samsung Research, by Samsung Electronics, by the Ministry of Education and Science of the Russian Federation (grant 14.756.31.0001) and by the NSERC Discovery
Grant RGPIN-2017-06936. 

\bibliographystyle{icml2017}
\setlength{\bibsep}{1.2mm plus 0.3ex}
{ \small
    \bibliography{references}}

\clearpage
\appendix

\newcommand{\toptitlebar}{
    \hrule height 4pt
    \vskip 0.25in
    \vskip -\parskip%
}
\newcommand{\bottomtitlebar}{
    \vskip 0.29in
    \vskip -\parskip
    \hrule height 1pt
    \vskip 0.09in%
}
\toptitlebar
\begin{center}
    {\LARGE\bf Supplementary Material (Appendix)}
    \\[4mm]
    {\LARGE\bf Quantifying Learning Guarantees for Convex but Inconsistent Surrogates}
\end{center}
\bottomtitlebar

\section*{Outline} 
\begin{description}[itemsep=0mm,topsep=0cm,leftmargin=*,font=\normalfont]
    \renewcommand\labelitemi{--}
    \item[Section~\ref{sec:techLemmas}:] Proofs of the two technical lemmas used in  Theorem~\ref{th:lowerBoundCalibrationFunctionInConsistent}.
    \item[Section~\ref{sec:LowerBoundThmProof}:] Proof of Theorem~\ref{th:lowerBoundCalibrationFunctionInConsistent}, which is the main result of this paper.
    \item[Section~\ref{sec:treeProofs}:] Lower bound on the calibration function for the tree-structured loss.
    \item[Section~\ref{sec:rankingProofs}:] Derivations for the mean average precision loss.
\end{description}

\section{Technical Lemmas}
\label{sec:techLemmas}
In this section, we prove two technical lemmas that are used in the proofs of the main theoretical claims of the paper. These two lemmas are the generalizations of the two corresponding lemmas of~\cite{osokin17consistency}.

Lemma~\ref{th:quadSurrogateExcess} computes the excess of the weighted surrogate risk~$\excess\surrogateweighted$ for the quadratic loss~$\surrogatelossquad$~\eqref{eq:quadrLoss}, which is central to our analysis presented in Section~\ref{sec:lowerBound}.
Lemma~\ref{th:quadSurrogateExcess} generalizes Lemma~9 of~\cite{osokin17consistency} by removing the assumption of $\colspace(\lossmatrix) \subseteq \scoresubset$.
Analogously to Lemma~9~\cite{osokin17consistency}, the key property of this result is that the excess~$\excess\surrogateweighted$ is jointly convex w.r.t.\ the parameters~$\scoreparamv$ and conditional distribution~$\qv$, which allows further analysis.

Lemma~\ref{th:breakingSymmetries} allows to cope with the combinatorial aspect of the calibration function computation.
In particular, when the excess of the weighted surrogate risk is convex, Lemma~\ref{th:breakingSymmetries} reduces the computation of the calibration function to a set of convex optimization problems, which often can be solved analytically.
Note that our Lemma~\ref{th:breakingSymmetries} is slightly different from Lemma 10 of~\citet{osokin17consistency} to deal with the difference of the definition of the excess population risk~\eqref{eq:excessRisks}.
\begin{lemma}
    \label{th:quadSurrogateExcess}
    Consider the quadratic surrogate $\surrogatelossquad$~\eqref{eq:quadrLoss} defined for a task loss~$\lossmatrix$.
    Let a subspace of scores~$\scoresubset \subseteq \R^\outputvarcard$ be parametrized by $\scoreparamv\in \R^\scoresubspacedim$, i.e., $\scorev = \scorematrix\scoreparamv \in \scoresubset$ with $\scorematrix \in \R^{\outputvarcard \times \scoresubspacedim}$.
    Then, the excess of the weighted surrogate loss can be expressed as
    \begin{equation*}
    \excess\surrogateweightedquad(\scorematrix\scoreparamv, \qv) := 
    \surrogateweightedquad(\scorematrix\scoreparamv, \qv)-\inf_{\scoreparamv' \in \R^\scoresubspacedim} \surrogateweightedquad(\scorematrix\scoreparamv', \qv) =
    \frac{1}{2\outputvarcard}\| \scorematrix\scoreparamv + \proj_{\scoresubset} \lossmatrix \qv \|_2^2,
    \end{equation*}
    where $\proj_{\scoresubset} := \scorematrix (\scorematrix^\transpose \scorematrix)^\pinv \scorematrix^\transpose$ is the orthogonal projection on the subspace~$\scoresubset = \colspace(\scorematrix)$.
\end{lemma}
\begin{proof}
    The proof is almost identical to the proof of Lemma~9 of~\cite{osokin17consistency} generalizing it only in the last equality.
    By the definition of the quadratic surrogate~$\surrogatelossquad$~\eqref{eq:quadrLoss}, we have
    \begin{align*}
    \surrogateweighted(\scorev(\scoreparamv), \qv) &= \frac{1}{2\outputvarcard}(\scoreparamv^\transpose \scorematrix^\transpose \scorematrix \scoreparamv + 2 \scoreparamv^\transpose \scorematrix^\transpose \lossmatrix \qv) + r(\qv), \\
    \scoreparamv^* &:= \argmin\nolimits_\scoreparamv \surrogateweighted(\scorev(\scoreparamv), \qv) =  -(\scorematrix^\transpose \scorematrix)^\pinv \scorematrix^\transpose \lossmatrix \qv, \\
    \excess\surrogateweighted(\scorev(\scoreparamv), \qv) &= \frac{1}{2\outputvarcard}(\scoreparamv^\transpose \scorematrix^\transpose \scorematrix \scoreparamv + 2 \scoreparamv^\transpose \scorematrix^\transpose \lossmatrix \qv 
    + \qv^\transpose \lossmatrix^\transpose \scorematrix (\scorematrix^\transpose \scorematrix)^\pinv \scorematrix^\transpose \lossmatrix \qv) \\
    & = \frac{1}{2\outputvarcard} \| \scorematrix\scoreparamv + \proj_{\scoresubset} \lossmatrix \qv \|_2^2,
    \end{align*}
    where $r(\qv)$ denotes the quantity independent of the parameters~$\scoreparamv$.
    Note that if the assumption $\colspace(\lossmatrix) \subseteq \colspace(\scorematrix)$ holds we have 
    $
    \proj_{\scoresubset} \lossmatrix = \lossmatrix,
    $
    which is the statement of Lemma~9~\cite{osokin17consistency}.
\end{proof}

\begin{lemma}
    \label{th:breakingSymmetries}
    For any task loss~$\lossmatrix$, a surrogate loss~$\surrogateloss$ that is continuous and bounded from below, and a set of scores~$\scoresubset$, the calibration function can be lower bounded as 
    \begin{equation}
    \label{eq:calibrationfunc:decomposition}
    \calibrationfunc_{\surrogateloss,\lossmatrix,\scoresubset}(\eps) \geq \min_{i\neq j} \calibrationfunc_{ij}(\eps),
    \end{equation}
    where $\calibrationfunc_{ij}$ is defined via minimization of the same objective as~\eqref{eq:calibrationfunc}, but w.r.t.\ a smaller domain:
    \begin{align}
    \label{eq:calibrationfunc:generalBreakingSymmetries}
    \calibrationfunc_{ij}(\eps)
    =
    \inf_{\scorev, \qv} \;& \excess\surrogateweighted(\scorev, \qv), \\
    \notag %
    \text{s.t.} \;& \lossweighted_i(\qv) \leq\lossweighted_j(\qv) - \eps,\\
    \notag %
    &\lossweighted_i(\qv) \leq \lossweighted_c(\qv),\;\; \forall c \in \outputdomain, \\
    \notag %
    &\score_j \geq \score_c,\;\:\: \forall c \in \outputdomain, \\
    \notag %
    & \scorev \in \scoresubset,\\
    \notag
    & \qv \in \simplex_\gtvarcard.
    \end{align}
    Here $\lossweighted_c(\qv) := (\lossmatrix \qv)_c$ is the expected loss if predicting label~$c$.
    The index~$i$ represents a label with the smallest expected loss while the index~$j$ represents a label with the largest score.
\end{lemma}
\begin{proof}
    We use the notation $\scoresubset_j$ to define the set of score vectors~$\scorev$ where the predictor $\predictor(\scorev)$ takes the value~$j$, i.e., $\scoresubset_j := \{\scorev \in \scoresubset \mid \predictor(\scorev) = j\}$.
    The union of the sets $\scoresubset_j$, $j \in \outputdomain$, equals the whole set~$\scoresubset$.
    Sets $\scoresubset_j$ might not contain their boundaries because of the usage of a particular tie-breaking strategy, thus we consider the sets $\overline{\scoresubset}_j := \{\scorev \in \scoresubset \mid \score_j \geq  \score_c, \forall c \in \outputdomain\}$, which are the closures of $\scoresubset_j$ if $\scoresubset_j$ are not empty.
    It also might happen that because of a particular tie-breaking strategy a set~$\scoresubset_j$ is empty, while the corresponding $\overline{\scoresubset}_j$ is not.
    
    If $\scorev \in \scoresubset_j$, i.e. $j=\predictor(\scorev)$, then the feasible set of  probability vectors~$\qv$ for which a label~$i$ is one of the best possible predictions (i.e. $\excess\lossweighted(\scorev, \qv) = \ell_j(\qv) - \ell_i(\qv)  \geq \eps$) equals $$\simplex_{\gtvarcard,i,j,\eps} := \{ \qv \in \simplex_{\gtvarcard} \mid \lossweighted_i(\qv) \leq \lossweighted_c(\qv), \forall c \in \outputdomain; \ell_j(\qv) - \ell_i(\qv) \geq \eps\},$$ because $\inf_{\scorev' \in \R^\outputvarcard} \lossweighted(\scorev', \qv) = \min_{c \in \outputdomain} \lossweighted_c(\qv)$.
    
    The union of the sets~$\{\scoresubset_j \times \simplex_{\gtvarcard,i,j,\eps}\}_{i, j \in \outputdomain, i\neq j}$ exactly equals the feasibility set of the optimization problem~\eqref{eq:calibrationfunc}-\eqref{eq:calibrationfunc:epsConstr} (note that this is not true for the union of the sets~$\{\overline{\scoresubset}_j \times \simplex_{\gtvarcard,i,j,\eps}\}_{i, j \in \outputdomain, i\neq j}$, which can be strictly larger), thus we can rewrite the definition of the calibration function as follows:
    \begin{equation}
    \label{eq:calibrationfunc:proof}
    \calibrationfunc_{\surrogateloss,\lossmatrix,\scoresubset}(\eps) = \min_{\substack{i, j \in \outputdomain \\ i\neq j}} \inf_{\substack{\scorev \in \scoresubset_j,\\ \qv \in \simplex_{\gtvarcard,i,j,\eps}}} \excess\surrogateweighted(\scorev, \qv)
    \geq
    \min_{\substack{i, j \in \outputdomain \\ i\neq j}} \inf_{\substack{\scorev \in \overline{\scoresubset}_j,\\ \qv \in \simplex_{\gtvarcard,i,j,\eps}}} \excess\surrogateweighted(\scorev, \qv)
    =
    \min_{i\neq j} \calibrationfunc_{ij}(\eps),
    \end{equation}
    which finishes the proof.
    Note that the inequality of~\eqref{eq:calibrationfunc:proof} can be not tight only if some~$\scoresubset_j$ is empty, but the corresponding  $\overline{\scoresubset}_j$ is not (due to continuity of the function~$\excess\surrogateweighted(\scorev, \qv)$, which follows from Lemma~27 of~\citep{zhang04}).
\end{proof}

\section{Proof of Theorem~\ref{th:lowerBoundCalibrationFunctionInConsistent}} \label{sec:LowerBoundThmProof}
\begin{reptheorem}{th:lowerBoundCalibrationFunctionInConsistent}[Lower bound on the calibration function $\calibrationfunc_{\surrogatelossquad,\lossmatrix,\scoresubset}(\eps)$]
    \label{th:rep:lowerBoundCalibrationFunctionInConsistent}
    For any task loss~$\lossmatrix$, its quadratic surrogate~$\surrogatelossquad$, and a score subspace~$\scoresubset$, the calibration function is bounded from below:
    \begin{equation}
    \label{eq:rep:lowerBoundcalibrationFunctionInconsistent}
    \calibrationfunc_{\surrogatelossquad,\lossmatrix,\scoresubset}(\eps)
    \geq
    \min_{i\neq j} 
    \max_{v \geq 0} \frac{(\eps v - \xi_{ij}(v))_+^2}{2\outputvarcard \|\proj_{\scoresubset} \Delta_{ij}\|_2^2},\quad 
    \text{where\;\;\;} \xi_{ij}(v) := \bigl\| \:\lossmatrix^\transpose (v \id_\outputvarcard - \proj_{\scoresubset}) \Delta_{ij} \: \bigr\|_\infty,
    \end{equation}
    $\proj_{\scoresubset}$ is the orthogonal projection on the subspace~$\scoresubset$, $(x)_+^2 := [x > 0] x^2$ is the truncation of the parabola to its right branch  and $\Delta_{ij} := \unit_i - \unit_j \in \R^k$ with $\unit_c\in \R^k$ being the $c$-th column of the identity matrix $\id_\outputvarcard$.
    By convention, if both numerator and denominator of~\eqref{eq:rep:lowerBoundcalibrationFunctionInconsistent} equal zero the whole bound equals zero.
    If only the denominator equals zero then the whole bound equals infinity (the particular pair of~$i$ and~$j$ is effectively not considered).
\end{reptheorem}
\begin{proof}
    First, let us assume that the score subspace $\scoresubset$ is defined as the column space of a matrix~$\scorematrix \in \R^{\outputvarcard \times \scoresubspacedim}$, i.e., $\scorev(\scoreparamv)=\scorematrix\scoreparamv$.
    For technical convenience, we can also assume that~$\scorematrix$ is of the full rank, $\rank(\scorematrix) = \scoresubspacedim$.
    Lemma~\ref{th:quadSurrogateExcess} gives us the expression~$\excess\surrogateweightedquad(\scorematrix\scoreparamv, \qv)  =
    \frac{1}{2\outputvarcard}\| \scorematrix\scoreparamv + \proj_{\scoresubset} \lossmatrix \qv \|_2^2$ for the excess surrogate, which is jointly convex w.r.t.\ a conditional probability vector~$\qv$ and parameters~$\scoreparamv$.
    
    The optimization problem~\eqref{eq:calibrationfunc}-\eqref{eq:calibrationfunc:epsConstr} is non-convex because the constraint~\eqref{eq:calibrationfunc:epsConstr} on the excess risk depends of the predictor function~$\predictor(\scorev)$, see Eq.~\eqref{eq:predictor}, containing the~$\argmax$ operation.
    However, if we constrain the predictor to output label $j$, i.e., $\score_j \geq \score_c$, $\forall c$, and the label delivering the smallest possible expected loss to be $i$, i.e., $(\lossmatrix \qv)_i \leq (\lossmatrix \qv)_c$, $\forall c$, the problem becomes convex because all the constraints are linear and the objective is convex.
    Lemma~\ref{th:breakingSymmetries} in Appendix~\ref{sec:techLemmas} allows to bound the calibration function with the minimal w.r.t.\ selected labels $i$ and $j$ optimal value of one of the convex problems, i.e.,  $\calibrationfunc_{\surrogatelossquad,\lossmatrix,\scoresubset}(\eps) \geq \min\limits_{i \neq j} \calibrationfunc_{ij}(\eps)$, %
    where~$\calibrationfunc_{ij}(\eps)$ is defined as follows:
    \begin{align}
    \label{eq:calibrationfunc:breakingSymmetries}
    \calibrationfunc_{ij}(\eps)
    =
    \min_{\scoreparamv, \qv} \;& \frac{1}{2 \outputvarcard}\|\scorematrix \scoreparamv + \proj_{\scoresubset}\lossmatrix \qv\|_2^2, \\
    \notag %
    \text{s.t.} \;& (\lossmatrix \qv)_i \leq (\lossmatrix \qv)_j - \eps,\\
    \notag %
    &(\lossmatrix \qv)_i \leq (\lossmatrix \qv)_c,\;\; \forall c \in \outputdomain, \\
    \notag %
    &(\scorematrix \scoreparamv)_j \geq (\scorematrix \scoreparamv)_c,\; \forall c \in \outputdomain, \\
    \notag %
    & \qv \in \simplex_\gtvarcard.
    \end{align}

    To obtain a lower bound, we relax~\eqref{eq:calibrationfunc:breakingSymmetries} by removing some of the constraints and arrive at
    \begin{align}
    \label{eq:calibrationfunc:breakingSymmetries:stripped}
    \outputvarcard  \calibrationfunc_{ij}(\eps)
    \geq
    \min_{\scoreparamv, \qv} \;& \frac{1}{2}\|\scorematrix \scoreparamv + \proj_{\scoresubset} \lossmatrix \qv\|_2^2, \\
    \label{eq:calibrationfunc:breakingSymmetries:stripped:epsConstr}
    \text{s.t.} \;& \Delta_{ij}^\transpose \lossmatrix \qv \leq - \eps,\\
    \label{eq:calibrationfunc:breakingSymmetries:stripped:fConstr}
    &\Delta_{ij}^\transpose \scorematrix \scoreparamv \leq 0, \\
    \label{eq:calibrationfunc:breakingSymmetries:stripped:simplexSum}
    & \one_\gtvarcard^\transpose \qv = 1, \\
    \label{eq:calibrationfunc:breakingSymmetries:stripped:simplexSign}
    & \q_c \geq 0, \;\; c=1,\dots, \gtvarcard.
    \end{align}
    where 
    $\Delta_{ij}^\transpose \lossmatrix \qv = (\lossmatrix \qv)_i - (\lossmatrix \qv)_j$,
    $\Delta_{ij}^\transpose \scorematrix \scoreparamv = (\scorematrix \scoreparamv)_i - (\scorematrix \scoreparamv)_j$, 
    and $\Delta_{ij} = \unit_i - \unit_j \in \R^k$ with $\unit_c \in \R^k$ being a vector with $1$ at position~$c$ and zeros elsewhere.
    Note that the relaxation defined by the problem~\eqref{eq:calibrationfunc:breakingSymmetries:stripped}-\eqref{eq:calibrationfunc:breakingSymmetries:stripped:simplexSign} is tighter than the one used in the proof of Theorem 7~\cite[Eq. (25)-(27)]{osokin17consistency}, because the latter omitted the simplex constraints~\eqref{eq:calibrationfunc:breakingSymmetries:stripped:simplexSum}-\eqref{eq:calibrationfunc:breakingSymmetries:stripped:simplexSign}.
    
    We now explicitly build a dual problem to the QP~\eqref{eq:calibrationfunc:breakingSymmetries:stripped}-\eqref{eq:calibrationfunc:breakingSymmetries:stripped:simplexSign}.
    If we used the standard Lagrangian approach we would have to invert the matrix defining the objective, which is difficult.
    Instead we use the dual formulation of~\citet[Page 160]{dorn60duality}, which allows to build a dual without inverting any matrices.\footnote{Here we show the dual of~\citet{dorn60duality} for the exact combination of constraints we are using.
        In the dual formulation, $\vec{v}$ and $\vec{u}$ are the extra variables corresponding to the inequality and equality constraints, respectively.\\[-3mm]
        \begin{center}
            \begin{minipage}{.4\textwidth}
                \centering
                The primal problem\\[-5mm]
                \begin{align*}
                \min_{\qv \geq 0,\; \scoreparamv} &\; \frac12 \begin{pmatrix} \qv^\transpose & \!\!\!\!\!\scoreparamv^\transpose\end{pmatrix} \begin{pmatrix} 
                H_{\qv\qv} & H_{\qv\scoreparamv} \\
                H_{\qv\scoreparamv}^\transpose & H_{\scoreparamv\scoreparamv} 
                \end{pmatrix} \begin{pmatrix} \qv \\ \scoreparamv\end{pmatrix}, \\
                \mbox{s.t.}  &\; A_{\qv} \qv  + A_\scoreparamv\scoreparamv \geq \vec{b} \\
                &\; C_{\qv} \qv  + C_\scoreparamv\scoreparamv = \vec{d}
                \end{align*}
            \end{minipage}
            \begin{minipage}{.4\textwidth}
                \centering
                The dual problem\\[-5mm]
                \begin{align*}
                \max_{\qv\; \scoreparamv, \; \vec{v} \geq 0,\; \vec{u}} &-\frac12 \begin{pmatrix} \qv^\transpose & \!\!\!\!\!\scoreparamv^\transpose\end{pmatrix} \begin{pmatrix} 
                H_{\qv\qv} & H_{\qv\scoreparamv} \\
                H_{\qv\scoreparamv}^\transpose & H_{\scoreparamv\scoreparamv} 
                \end{pmatrix} \begin{pmatrix} \qv \\ \scoreparamv\end{pmatrix} + \vec{b}^\transpose \vec{v} + \vec{d}^\transpose \vec{u}, \\
                \mbox{s.t.}  &\; A_{\qv}^\transpose \vec{v} + C_{\qv}^\transpose \vec{u} - H_{\qv\qv} \qv - H_{\qv\scoreparamv}\scoreparamv \leq \0 \\
                &\; A_{\scoreparamv}^\transpose \vec{v} + C_{\scoreparamv}^\transpose \vec{u} - H_{\qv\scoreparamv}^\transpose \qv - H_{\scoreparamv\scoreparamv}\scoreparamv = \0
                \end{align*}
            \end{minipage}
    \end{center}}
    For the problem~\eqref{eq:calibrationfunc:breakingSymmetries:stripped}-\eqref{eq:calibrationfunc:breakingSymmetries:stripped:simplexSign}, this dual can be written as follows:
    \begin{align}
    \label{eq:calibrationfunc:breakingSymmetries:strippedDual}
    \outputvarcard \calibrationfunc_{ij}(\eps)
    \geq
    \max_{\scoreparamv, \qv, v_\scorematrix \geq 0, v_\lossmatrix \geq 0, u} \;& -\frac{1}{2}\|\scorematrix \scoreparamv + \proj_{\scoresubset} \lossmatrix \qv\|_2^2 + v_\lossmatrix \eps + u , \\
    \label{eq:calibrationfunc:breakingSymmetries:strippedDual:ineq}
    & -v_{\lossmatrix} \lossmatrix^\transpose \Delta_{ij} + u \one_{\gtvarcard} - \lossmatrix^\transpose \proj_{\scoresubset} \lossmatrix \qv - \lossmatrix^\transpose \scorematrix \scoreparamv \leq \0_{\gtvarcard}, \\
    \label{eq:calibrationfunc:breakingSymmetries:strippedDual:eq}
    & -v_{\scorematrix} \scorematrix^\transpose \Delta_{ij} - \scorematrix^\transpose \lossmatrix \qv - \scorematrix^\transpose \scorematrix \scoreparamv = \0_{\scoresubspacedim}.
    \end{align}
    From the equality~\eqref{eq:calibrationfunc:breakingSymmetries:strippedDual:eq}, we can express $\scorematrix^\transpose \lossmatrix \qv = -v_{\scorematrix} \scorematrix^\transpose \Delta_{ij} - \scorematrix^\transpose \scorematrix \scoreparamv$ and substitute it in the objective~\eqref{eq:calibrationfunc:breakingSymmetries:strippedDual} and inequality~\eqref{eq:calibrationfunc:breakingSymmetries:strippedDual:ineq}.
    Using the identities~$\proj_{\scoresubset} = \scorematrix (\scorematrix^\transpose \scorematrix)^{-1} \scorematrix^\transpose$ and~$\proj_{\scoresubset} \scorematrix = \scorematrix$, we can exclude variables $\scoreparamv$, $\qv$ and get a simpler bound.
    Note that this step leads to a valid lower bound because for any $v_\scorematrix \geq 0$ there exist feasible values of variables~$\qv$ and $\scoreparamv$ (we can take simply $\qv = \0$, $\scoreparamv = - v_{\scorematrix} (\scorematrix^\transpose \scorematrix)^{-1}  \scorematrix^\transpose \Delta_{ij}$).
    The new bound depends on the three variables only:
    \begin{align}
    \label{eq:calibrationfunc:breakingSymmetries:strippedDualSimple}
    \outputvarcard \calibrationfunc_{ij}(\eps)
    \geq
    \max_{v_\scorematrix \geq 0, v_\lossmatrix \geq 0, u} \;& - \frac{1}{2} v_{\scorematrix}^2 \Delta_{ij}^\transpose \proj_{\scoresubset} \Delta_{ij}  + v_\lossmatrix \eps + u, \\
    \label{eq:calibrationfunc:breakingSymmetries:strippedDualSimple:ineq}
    & -v_{\lossmatrix} \lossmatrix^\transpose \Delta_{ij} + u \one_{\gtvarcard} + v_{\scorematrix}\lossmatrix^\transpose \proj_{\scoresubset} \Delta_{ij} \leq \0_{\gtvarcard}.
    \end{align}

    First, consider the case $\Delta_{ij}^\transpose \proj_{\scoresubset} \Delta_{ij} = \| \proj_{\scoresubset} \Delta_{ij} \|_2^2 \neq 0$.
    Given that~$v_{\scorematrix} \geq 0$ we can change the variables by introducing $\hat{v}_{\scorematrix} := v_{\scorematrix} \| \proj_{\scoresubset} \Delta_{ij} \|_2^2$, $v := v_\lossmatrix / v_\scorematrix$, $\hat{u} := u / v_\scorematrix$ after which we get
    \begin{align}
    \label{eq:calibrationfunc:breakingSymmetries:strippedDualSimpleRename}
    \outputvarcard \calibrationfunc_{ij}(\eps)
    \geq
    \frac{1}{\| \proj_{\scoresubset} \Delta_{ij} \|_2^2}\max_{\hat{v}_\scorematrix \geq 0, \hat{v}_\lossmatrix \geq 0, \hat{u}} \;& -\frac{1}{2} \hat{v}_{\scorematrix}^2   + \hat{v}_{\scorematrix} (v \eps + \hat{u}), \\
    \label{eq:calibrationfunc:breakingSymmetries:strippedDualRename:ineq}
    & -v \lossmatrix^\transpose \Delta_{ij} + \hat{u} \one_{\gtvarcard} + \lossmatrix^\transpose \proj_{\scoresubset} \Delta_{ij} \leq \0_{\gtvarcard}.
    \end{align}
    The global minimum of this function w.r.t.\ the variable~$\hat{v}_\scorematrix$ can be found analytically: if $v \eps + \hat{u} \geq 0$ it equals $\frac{1}{2 \| \proj_{\scoresubset} \Delta_{ij} \|_2^2} (v \eps + \hat{u})^2$, and zero otherwise.
    The constraint~\eqref{eq:calibrationfunc:breakingSymmetries:strippedDualRename:ineq} on~$\hat{u}$ can be substituted with $\hat{u} =  - \bigl\| \:\lossmatrix^\transpose (v \id_\outputvarcard - \proj_{\scoresubset}) \Delta_{ij} \: \bigr\|_\infty =: -\xi_{ij}(v)$, because we always consider both~$\calibrationfunc_{ij}(\eps)$ and~$\calibrationfunc_{ji}(\eps)$ when bounding the calibration function.

    Now, consider the boundary case of 
    $\Delta_{ij}^\transpose \proj_{\scoresubset} \Delta_{ij} = \| \proj_{\scoresubset} \Delta_{ij} \|_2^2 = 0$.
    The problem~\eqref{eq:calibrationfunc:breakingSymmetries:strippedDualSimple}-\eqref{eq:calibrationfunc:breakingSymmetries:strippedDualSimple:ineq} becomes~$\frac12\max_{v \geq 0} v ( \eps + \min \lossmatrix^\transpose \Delta_{ij}) $ implying that the objective equals $0$ if $\eps + \min \lossmatrix^\transpose \Delta_{ij} \leq 0$.
    Otherwise, the objective equals $+\infty$, which corresponds to the in-feasibility of the constraint~\eqref{eq:calibrationfunc:breakingSymmetries:stripped:epsConstr} of the primal problem.
    Note that because we always consider both~$\calibrationfunc_{ij}(\eps)$ and~$\calibrationfunc_{ji}(\eps)$ when bounding the calibration function we can substitute $v \min (\lossmatrix^\transpose \Delta_{ij})$ with $-\xi_{ij}(v)$.
\end{proof}

\section{Lower Bound on the Calibration Function for the Tree-Structured Loss}  \label{sec:treeProofs}

In this section, we compute the lower bound on the calibration function for the tree-structured loss defined in Section~\ref{sec:hierarhicalLoss}.
\begin{figure}
    \begin{center}
        \includegraphics[width=80mm]{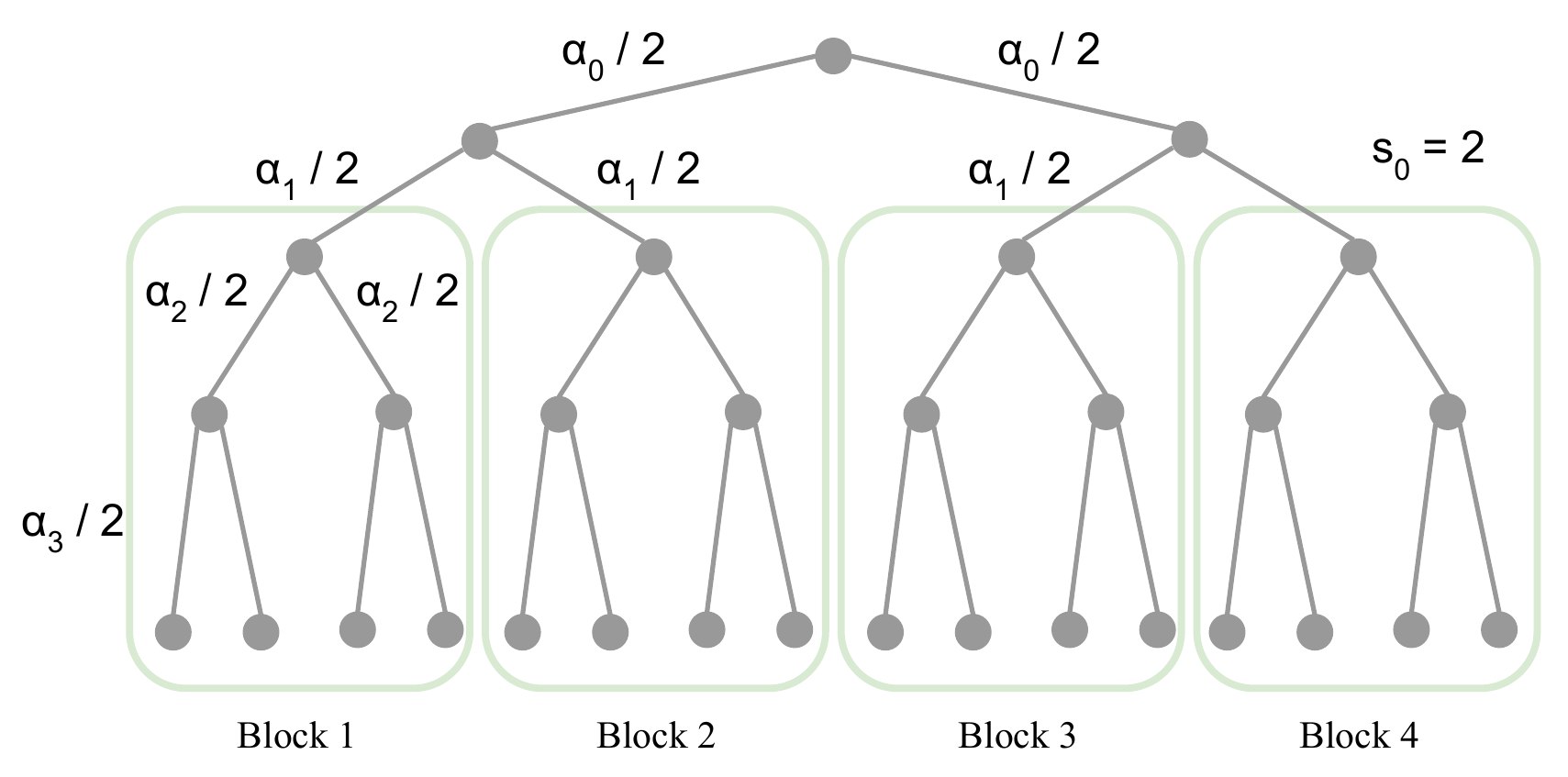}%
        \includegraphics[width=45mm]{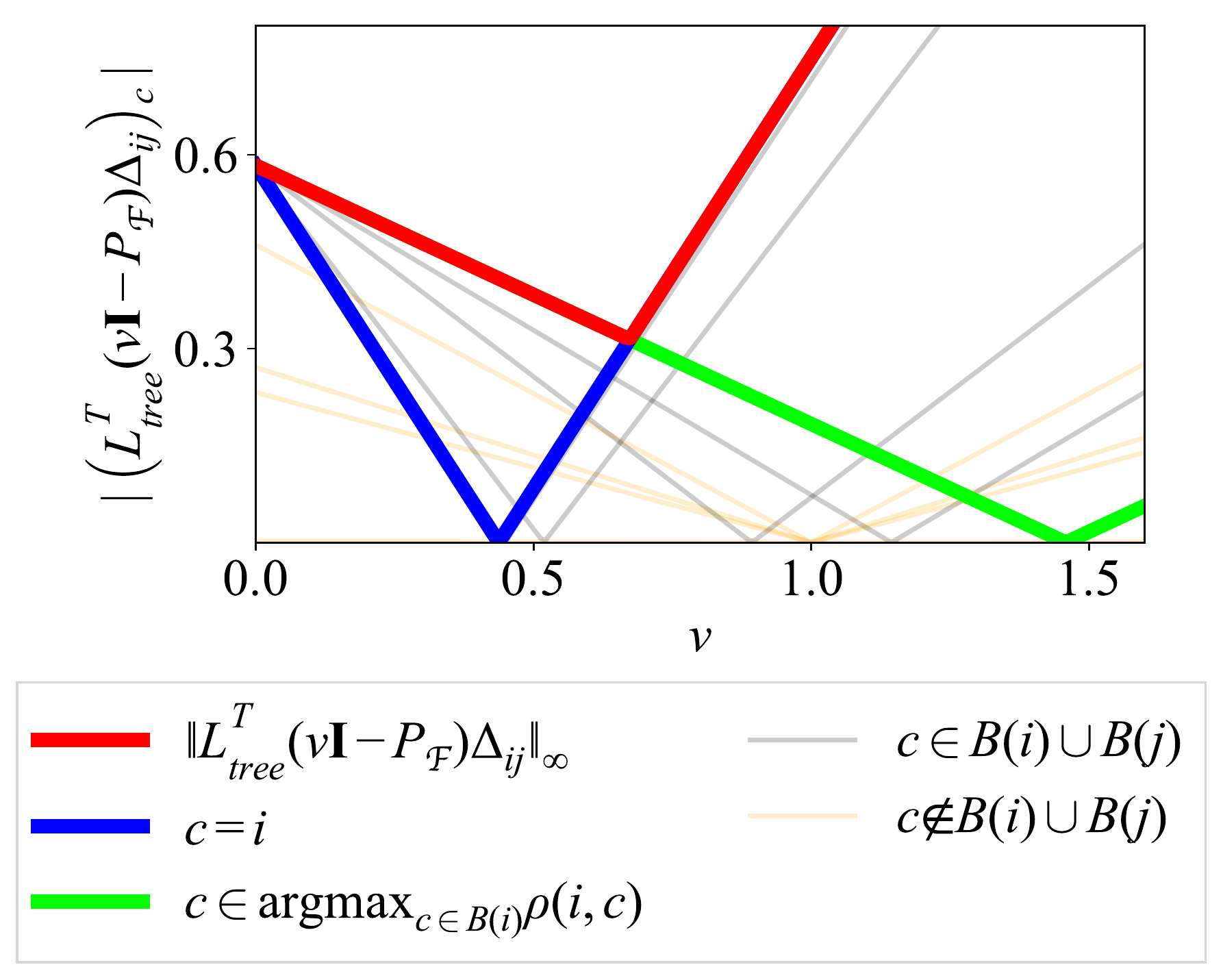}
    \end{center}
    \caption{\emph{Left:} An example of the tree-structured loss for the task of multi-class classification. \emph{Right:} Illustration of the proof of Lemma~\ref{th:treeLoss} (best viewed in color).
    The thin gray and brown lines show the absolute values of the components of the vector $\lossmatrix_{\text{\textup{tree}}}^\transpose (v \id - \proj_{\scoresubset_{\text{\textup{tree}}}}) \Delta_{ij}$ as functions of $v$.
    The bold blue and green lines correspond to the components at which the maximum value is achieved.
    The bold red line shows the resulting norm $\| \lossmatrix_{\text{\textup{tree}}}^\transpose (v \id - \proj_{\scoresubset_{\text{\textup{tree}}}}) \Delta_{ij} \|_\infty$. \label{fig:treeillustration}}
\end{figure}

\begin{lemma}\label{th:treeLoss}
For a particular consistency depth~$\consistencydepth$ and for the corresponding subspace $\scoresubset_{\text{\textup{tree}}, \consistencydepth}$, the projection operator~$\proj_{\scoresubset_\text{\textup{tree}}, \consistencydepth}$ at $\Delta_{ij}$ is computed as
\begin{equation}
\proj_{\scoresubset_\text{\textup{tree}}, \consistencydepth} \Delta_{ij} = \begin{cases}
0, & i \in B(j), \\
\frac{1}{|B(j)|} \left(\sum_{k \in B(i)} \unit_k - \sum_{k \in B(j)} \unit_k \right) & i \notin B(j).
\end{cases}
\end{equation}
The vectors $\xi_{ij}(v)$ are computed as
\begin{equation}
\xi_{ij}(v) = \begin{cases}
v \treedist(i, j) & i \in B(j), \\
\max\{|(v - 1) \treedist(i, j)|, |v(\treedist(i, j) - \blockdist) - (\treedist(i, j) - \avgblockdist)|\} & i \notin B(j),
\end{cases}
\end{equation}
for $\blockdist := \max_{c \in B(i)} \treedist(i, c)$ and $\avgblockdist := \frac{1}{|B(j)|} \sum_{c \in B(i)} \treedist(i, c)$. Finally, the following lower bound of the calibration function for the loss~$\lossmatrix_\text{\textup{tree}}$, its quadratic surrogate~$\surrogatelossquad$ and the score subspace~$\scoresubset_{\text{\textup{tree}}, \consistencydepth}$ holds:
\begin{equation}
\calibrationfunc_{\surrogatelossquad, \lossmatrix_{\text{\textup{tree}}}, \scoresubset_{\text{\textup{tree}}, \consistencydepth}}(\epsilon) \geq \lbrack \epsilon > \blockdist \rbrack \frac{(\minblockdist - \avgblockdist)^2}{(\minblockdist - \frac{\blockdist}{2})^2} \frac{(\epsilon - \frac{\blockdist}{2})_+^2}{4\numblocks},\label{eq:treeLowerBound}
\end{equation}
where $\minblockdist := \min_{c \notin B(i)} \treedist(i, c) > \blockdist$ and $\numblocks$ is the number of blocks when the tree is cut at the depth~$\consistencydepth$. 
\end{lemma}
\begin{proof}
For brevity, we shortcut the notation $\scorematrix_{\text{\textup{tree}}, \consistencydepth}$ to $\scorematrix$, $\scoresubset_{\text{\textup{tree}}, \consistencydepth}$ to $\scoresubset$ and $\lossmatrix_{\text{\textup{tree}}}$ to $\lossmatrix$.
First, we compute the projection operator $\proj_{\scoresubset} \unit_i$ and the lower-bound denominator $2 \outputvarcard \| \proj_{\scoresubset} \Delta_{ij} \|_2^2$.
Recall, that the subspace of allowed scores $\scoresubset$ defined as $\colspace{ \{ \sum_{l \in B(j)} \unit_l | j = 1, \dots, \outputvarcard \} }$ is of dimension~$\numblocks$.
The vector $\unit_i$ is orthogonal to the $\numblocks - 1$ different vectors $\sum_{l \in B(j)} \unit_l$, $j \not\in B(i)$, thus the projection $\proj_{\scoresubset} \unit_i$ equals the projection of $ \unit_i$ on the vector $\sum_{l \in B(i)} \unit_l$, which equals $\frac{1}{\blocksize} \sum_{l \in B(i)} \unit_l$ with $\blocksize:= B(i) = \frac{\outputvarcard}{\numblocks}$.
The projection square norm~$\|\proj_{\scoresubset} \Delta_{ij} \|_2^2$ equals $\frac{2}{\blocksize} = \frac{2\numblocks}{\outputvarcard}$ if $i \notin B(j)$ and $0$ if $i \in B(j)$.

Next, we compute $\xi_{ij}(v)$ defined as $\| \lossmatrix^\transpose (v I - \proj_\scoresubset) \Delta_{ij} \|_{\infty}$.
By definition of the loss function, the element of the loss matrix $\lossmatrix_{ci} = (\lossmatrix \unit_i)_c$ equals the tree distance from the leaf $i$ to the leaf $c$.
The projection operator~$\proj_{\scoresubset} \unit_i$ equals the vector $\frac{1}{\blocksize} \sum_{l \in B(i)} \unit_l$, therefore $(\lossmatrix \proj_\scoresubset \unit_i)_c$ is equal to the average tree distance from the elements of the block $B(i)$ to $c$: $\frac{1}{\blocksize} \sum_{l \in B(i)} \lossmatrix_{lc}$, which we denote by $\avgtreedist(i, c)$.
Note that the average distance~$\avgtreedist(i, c)$ is equal for all the leaves $c$ that belong to the same block~$B(c)$.
With this notation, we have the following equality:
\begin{equation}\label{eq:infnorm_treeproof}
\| \lossmatrix^\transpose (vI - \proj_\scoresubset) \Delta_{ij} \|_\infty = \max_{c \in \outputdomain} | v(\treedist(i,c) - \treedist(j, c)) - (\avgtreedist(i, c) - \avgtreedist(j, c)) |.
\end{equation}
On the right-hand side, each component is the absolute value of a linear in $v$ function, which equals zero at $v = \frac{\avgtreedist(i, c) - \avgtreedist(j, c)}{\treedist(i, c) - \treedist(j, c)}$ and the absolute value of the slope equals $|\treedist(i, c) - \treedist(j, c)|$.
We consider the cases when the labels $i$ and $j$ are in the same and different blocks separately. 

If $i$ and $j$ are in the same block we have that~$\avgtreedist(i, c) = \avgtreedist(j, c)$ and, by the reverse triangle inequality, $|\treedist(i, c) - \treedist(j, c)| \leq \treedist(i, j)$ with the equality holding for $c=i$ or $c=j$, which implies that  $\xi_{ij}(v) = \|(\lossmatrix^\transpose (vI - \proj_\scoresubset) \Delta_{ij} \|_\infty = v \treedist(i, j)$ for $i \in B(j)$. 

Now, we study the second case where $i$ and $j$ are in different blocks. We first show that 
\begin{equation}\label{eq:treeLemmaInequality}
|\avgtreedist(i, c_1) - \avgtreedist(j, c_1)| \leq |\avgtreedist(i, c_2) - \avgtreedist(j, c_2)| \quad \text{if $c_1 \notin B(i) \cup B(j)$ and $c_2 \in B(i) \cup B(j)$}.
\end{equation}
This inequality is crucial for the proof and holds due to the restriction on tree weights and node degrees.

In this paragraph, we will show that the left-hand side of the inequality~\eqref{eq:treeLemmaInequality} achieves its maximum when $c_1\notin B(i) \cup B(j)$ is in the block closest to $B(j)$ (or, due to the loss symmetries, in the block closest to $B(i)$).
If the lowest common ancestor of $i$ and $j$ is not an ancestor of $c_1$ the difference of the average distances equals zero due to equality of the paths from the lowest common ancestor to $i$ and $j$.
Otherwise, there exists $c_1 \notin B(i) \cup B(j)$ such that the lowest common ancestor of $i$ and $j$ is an ancestor of $c_1$.
Then, $\avgtreedist(j, c_1)$ is minimized and  $\avgtreedist(i, c_1)$ is simultaneously maximized for a component $c_1$ closest to the block $B(j)$.
In this case, the left-hand side maximum value equals $\avgtreedist(i, j) - \min_{c \notin B(j)} \avgtreedist(j, c)$ because $\avgtreedist(i, j)=\avgtreedist(i, c_1)$.

The right-hand side of the inequality~\eqref{eq:treeLemmaInequality} is the same for any choice of $c_2 \in B(i) \cup B(j)$ and is equal to $\avgtreedist(i, j) - \avgtreedist(j, c_2)$ for some $c_2 \in B(j)$.
Since the average distance within the block is smaller than the average distance to any node outside of the block, i.e., $\min_{c \notin B(j)} \avgtreedist(j, c) \geq \avgtreedist(j, c_2)$ for $c_2 \in B(j)$, the inequality~\eqref{eq:treeLemmaInequality} holds. The same arguments also show that 
\begin{equation} \label{eq:treeLemmaInequality2}
|\treedist(i, c_1) - \treedist(j, c_1)| \leq |\treedist(i, c_2) - \treedist(j, c_2)| \quad \text{if $c_1 \notin B(i) \cup B(j)$ and $c_2 \in B(i) \cup B(j)$}.
\end{equation}

Recall that in our case the infinity norm in~\eqref{eq:infnorm_treeproof} equals the component-wise maximum of the absolute values of the linear functions of $v$.
We will show below that for a small enough $v$ the maximum is achieved at the components that have the smallest slope $|\treedist(i, c) - \treedist(j, c)|$ among the ones with the largest offset $|\avgtreedist(i, c) - \avgtreedist(j, c)|$ and from some point for larger values of $v$ the maximum is achieved at the components with the steepest slope (see Figure~\ref{fig:treeillustration} right for the illustration).

Consider a leaf $c_2 \in B(i)$ farthest from the leaf $i$, i.e., $c_2 \in \argmax_{c \in B(i)} \treedist(i, c)$ (defines the green line in Figure~\ref{fig:treeillustration} right).
The offset $|\avgtreedist(i, c_2) - \avgtreedist(j, c_2)|$ is the same for all $c_2 \in B(i)$ and, by~\eqref{eq:treeLemmaInequality}, is larger than the offsets of the components $c_1 \notin B(i) \cup B(j)$.
The slope $|\treedist(i, c_2) - \treedist(j, c_2)|$ is the smallest among the components in $B(i) \cup B(j)$.
The component $c_2$ of $\lossmatrix^\transpose (v \id - \proj_\scoresubset)\Delta_{ij}$ equals zero for $v^*_{c_2} := \frac{\avgtreedist(j, c_2) - \avgtreedist(i, c_2)}{\treedist(j, c_2) - \treedist(i, c_2)} = \frac{\treedist(i, j) - \avgtreedist(i, c_2)}{\treedist(i, j) - \treedist(i, c_2)}$, where $v^*_{c_2} > 1$ by definition of~$c_2$, i.e., because $\treedist(i, c_2)$ is the maximal distance, which is not smaller than the average  distance  $\avgtreedist(i, c_2)$.
Finally, for $v \leq 1$ this component has higher values than the values of the components $c \notin B(i) \cup B(j)$.
Indeed, the latter are equal to zero at $v = 1$ and have smaller offset at $v = 0$ (thin brown lines in Figure~\ref{fig:treeillustration} right for $v < 1$).

The component $i$ of $\lossmatrix^\transpose (v \id - \proj_\scoresubset) \Delta_{ij}$ has the steepest slope $| \treedist(i, j) - \treedist(i, i)| = \treedist(i, j)$ and the same offset as in the previous paragraph $|\avgtreedist(i, i) - \avgtreedist(i, j)| = |\avgtreedist(i, c_2) - \avgtreedist(j, c_2)|$ (defines the blue line in Figure~\ref{fig:treeillustration} right).
The component equals zero for $v^*_{i} := \frac{\avgtreedist(j, i) - \avgtreedist(i, i)}{\treedist(j, i) - \treedist(i, i)} = \frac{\treedist(j, i) - \avgtreedist(i, i)}{\treedist(j, i)}$, where $v^*_{i} \leq 1$.
As a result, the component~$i$ has higher values than the components $c \notin B(i) \cup B(j)$, since they have smaller slope $|\treedist(i, c) - \treedist(i, c)|$ (due to the inequality \eqref{eq:treeLemmaInequality2}) and equal zero for $v = 1$ (thin brown lines in Figure~\ref{fig:treeillustration} right for $v > 1$).

Since all the components $c \in B(i) \cup B(j)$ have the same offset, the maximum is achieved either at~$c_2$ or at $i$:
\begin{equation}
\| \lossmatrix^\transpose(v \id - \proj_\scoresubset) \Delta_{ij} \|_\infty = \max\{|v\treedist(i, j) - (\treedist(i, j) - \avgblockdist)|, |v(\treedist(i, j) - \blockdist) - (\treedist(i, j) - \avgblockdist)|\},
\end{equation}
where $\blockdist := \treedist(i, c_2)$ is the maximal distance within a block and $\avgblockdist := \avgtreedist(i, c_2)$ is the average distance within a block.

Next, we compute $\max_{v \geq 0} (\epsilon v - \xi_{ij}(v))^2_+$.
If $i$ and $j$ are in the same block we have $(\epsilon v - \xi_{ij}(v))^2_+ = (v(\epsilon - \treedist(i, j)))^2_+$, which equals zero for $\epsilon \leq \treedist(i, j)$ and $+\infty$ otherwise. 
If $i$ and $j$ are in different blocks we have the maximum value of $(\epsilon v - \xi_{ij}(v))^2_+$ equal to $+\infty$ when $\epsilon > \treedist(i, j)$.
In the case when $\epsilon \leq \treedist(i, j)$, the maximum is achieved at the intersection point $v \treedist(i, j) - (\treedist(i, j) - \avgblockdist) = -v (\treedist(i, j) - \blockdist) + (\treedist(i, j) - \avgblockdist)$, $v=\frac{2(\treedist(i, j) - \avgtreedist)}{2\treedist(i, j) - \blockdist}$.
The maximum value is positive if and only if $\epsilon > \frac{\blockdist}{2}$, so for $\eps \leq \treedist(i, j)$ we obtain
\begin{equation}
\max_{v \geq 0} (\epsilon v - \xi_{ij}(v))^2_+ = \frac{(\treedist(i, j) - \avgblockdist)^2}{(\treedist(i, j) - \frac{\blockdist}{2})^2}(\epsilon - \frac{\blockdist}{2})_+^2
\end{equation}
and~$+\infty$ otherwise.

Finally, to get the actual lower bound on the calibration function, we compute the minimum with respect to all labels $\min_{i \neq j} \max_{v \geq 0} (\epsilon v - \xi_{ij}(v))^2_+$.
When $i$ and $j$ are in the same block, they deliver minimum value $0$ for $\varepsilon \leq \treedist(i, j)$ and the maximum value of $\treedist(i, j)$ within a block equals $\blockdist$ by definition of~$\blockdist$.
For $\eps > \blockdist$, the minimum is delivered by $i$ and $j$ in different blocks.
For the average distance within the block, we have $\avgblockdist \geq \frac{\blockdist}{2}$ for the trees with the number of children and the weights of edges equal at the same depth level, therefore the outer minimum w.r.t.~$i$ and $j$ is achieved at the smallest distance between two blocks $\minblockdist := \min_{i \notin B(j)} \treedist(i, j) > \blockdist$. As a result, we obtain the bound
\begin{equation}
\calibrationfunc_{\surrogatelossquad, \lossmatrix_{\text{\textup{tree}}}, \scoresubset_{\text{\textup{tree}}, \consistencydepth}}(\epsilon) \geq \lbrack \epsilon > \blockdist \rbrack \frac{(\minblockdist - \avgblockdist)^2}{(\minblockdist - \frac{\blockdist}{2})^2} \frac{(\epsilon - \frac{\blockdist}{2})_+^2}{4\numblocks},
\end{equation}
which completes the proof.
\end{proof}

In the next lemma, we compute the quantities $\blockdist$, $\avgblockdist$, $\minblockdist$ using the tree weights $\{ \frac12\treeweight_s \}_{s=0}^{\treedepth - 1}$ to finish the computation of the bound~\eqref{eq:consistencypaperboundgeneralized} of the main paper.
\begin{lemma}\label{th:lemmaTree}
For a particular consistency depth $\consistencydepth$ and the corresponding subspace $\scoresubset_{\text{\textup{tree}}, \consistencydepth}$, the maximum distance within an arbitrary block $\blockdist_\consistencydepth$, the minimum distance between a leaf in a block and a leaf outside the block $\minblockdist_\consistencydepth$ and the average distance within a block $\avgblockdist_\consistencydepth$ can be computed as follows:
\begin{align}
\blockdist_\consistencydepth &= \max_{i \in B(j)} \treedist(i, j) = \sum_{s = \consistencydepth}^{\treedepth - 1} \treeweight_s \\
\minblockdist_\consistencydepth &= \min_{i \notin B(j)} \treedist(i, j) = \sum_{s = \consistencydepth - 1}^{\treedepth - 1} \treeweight_s \\
\avgblockdist_\consistencydepth &= \frac{1}{|B(j)|} \sum_{i \in B(j)} \treedist(i, j) = \sum_{s=\consistencydepth}^{\treedepth - 1} \alpha_s \frac{ (\prod\nolimits_{s'=\consistencydepth}^{s} \nchildren_{s'} ) - 1}{\prod\nolimits_{s'=\consistencydepth}^{s} \nchildren_{s'}}.
\end{align}
\end{lemma}
\begin{proof}
The expressions for $\blockdist_\consistencydepth$ and $\minblockdist_\consistencydepth$ immediately follow from the definition of the distance $\treedist(i, j)$. 

To obtain the expression for $\avgtreedist_\consistencydepth$, we rewrite the distance $\treedist(i, j)$ between leaves $i$ and $j$ in the same block $B(j)$ as the weighted sum of indicators:
\begin{equation}
\treedist(i, j) = \sum_{s = \consistencydepth}^{\treedepth - 1} \treeweight_s [\text{path from $i$ to $j$ contains an edge of depth $s$}].
\end{equation}
Then, we fix a leaf $j$ and compute the number of paths from $j$ to the leaves in $B(j)$ that contain an edge of depth $s$.
Such paths go through the same node (the ancestor of $j$ at the depth $s$) on the way up from node $j$ and go through one of $\prod_{s' = \consistencydepth}^s \nchildren_{s'} - 1$ possible nodes at the depth~$s$ on the way down.
From each node of depth $s$, the path can further go to one of $\prod_{s'={s+1}}^{\treedepth - 1} \nchildren_{s'}$ leaves on the way down.
Therefore, there are $\left(\prod_{s' = \consistencydepth}^s \nchildren_{s'} - 1\right)\left(\prod_{s' = s + 1}^{\treedepth - 1} \nchildren_{s'}\right)$ paths that contain an edge of depth $s$.

Next, we rewrite $\sum_{i \in B(j)} \treedist(i, j)$ using the indicator notation and compute the sum:
\begin{align}
\sum_{i \in B(j)} \treedist(i, j) &= \sum_{s=\consistencydepth}^{\treedepth - 1} \sum_{i \in B(j)} \alpha_s [\text{path from $j$ to $i$ contains an edge of depth $s$}] \\
&= \sum_{s=\consistencydepth}^{\treedepth - 1} \treeweight_s\left(\prod_{s' = \consistencydepth}^s \nchildren_{s'} - 1\right)\left(\prod_{s' = s + 1}^{\treedepth - 1} \nchildren_{s'}\right).
\end{align}
Since the number of leaves in a block is $\prod_{s'=\consistencydepth}^{\treedepth - 1} \nchildren_{s'}$, we have
\begin{equation}
\avgblockdist_\consistencydepth = \frac{1}{|B(j)|} \sum_{i \in B(j)} \treedist(i, j) = \sum_{s=\consistencydepth}^{\treedepth - 1} \alpha_s \frac{ (\prod\nolimits_{s'=\consistencydepth}^{s} \nchildren_{s'} ) - 1}{\prod\nolimits_{s'=\consistencydepth}^{s} \nchildren_{s'}},
\end{equation}
which finishes the proof.
\end{proof}
Note that for tree-depth $\treedepth=2$ the minimum $\minblockdist_1$ equals $\treeweight_0 + \treeweight_1 = 1$. As a result, our calibration function lower bound coincides with the exact calibration function from \cite{osokin17consistency}.

\section{Derivations for the Mean Average Precision Loss} \label{sec:rankingProofs}
In this section, we prove several statements about $\scorematrix_{\text{\textup{sort}}}$ and $\lossmatrix_{\text{\textup{mAP}}}$, which are used in Section~\ref{sec:rankingLosses}.
\begin{lemma}\label{th:rankingprojection}
The matrix $\scorematrix_{\text{\textup{sort}}}^\transpose \scorematrix_{\text{\textup{sort}}}$ has the following form:
\begin{equation}
(\scorematrix_{\text{\textup{sort}}}^\transpose \scorematrix_{\text{\textup{sort}}})_{pq} = \begin{cases} 
(\size - 1)! H_{\size, 2}, & p = q, \\
(\size - 2)! (H_{\size, 1}^2 - H_{\size, 2}), & p \neq q,
\end{cases}
\end{equation}
where $H_{n, m} := \sum_{k=1}^n \frac{1}{k^m}$ is the generalized harmonic number of order $m$ of $n$. As a result, for distinct permutations $\pi$ and $\omega$, the square norm of the projection is equal to
\begin{equation}
\| \proj_{\scoresubset_{\text{\textup{sort}}}} \Delta_{\pi \omega} \|_2^2 = \frac{1}{(\size - 2)! (\size H_{\size, 2} - H_{\size, 1}^2)} \sum_{p=1}^{\size} \left(\frac{1}{\pi(p)} - \frac{1}{\omega(p)}\right)^2.
\end{equation}
The condition number~$\condnum(\scorematrix_{\text{\textup{sort}}})$ equals~$\frac{\sqrt{\size-1} H_{\size, 1}}{\sqrt{ \size H_{\size, 2} - H_{\size, 1}^2}\vphantom{(\big(\bigr)}}$.
\end{lemma}
\begin{proof}
By definition, $(\scorematrix_{\text{\textup{sort}}}^\transpose \scorematrix_{\text{\textup{sort}}})_{pq} = \sum_{\sigma \in S_\size} \frac{1}{\sigma(p) \sigma(q)}$. We can rewrite the sum as the sum over the permutations with fixed values $\sigma(p)$ and $\sigma(q)$ and then sum over the fixed values. Therefore, the sum is equal to $(\size - 1)! H_{\size, 2}$ when $p=q$ and is equal to $(\size - 2)! (H_{\size, 1}^2 - H_{\size, 2})$ otherwise.

We now have $\scorematrix_{\text{\textup{sort}}}^\transpose \scorematrix_{\text{\textup{sort}}} = (\size - 2)! (\size H_{\size, 2} - H_{\size, 1}^2) \id_{\size} + (\size - 2)! (H_{\size, 1}^2 - H_{\size, 2}) {\bf 1}{\bf 1}^\transpose$. The Sherman-Woodbury formula for the matrix inversion gives us the sum of the scalar matrix $\frac{1}{(\size - 2)! (\size H_{\size, 2} - H_{\size, 1})} \id_{\size}$ and the constant matrix. Since ${\bf 1}^{\transpose} \scorematrix_{\text{\textup{sort}}}^\transpose \Delta_{\pi\omega} = \sum_{p=1}^{\size} \left(\frac{1}{\pi(p)} - \frac{1}{\omega(p)} \right) = 0$, the square norm of the projection equals $\frac{1}{(\size - 2)! (\size H_{\size, 2} - H_{\size, 1}^2)} \Delta_{\pi\omega}^\transpose \scorematrix_{\text{\textup{sort}}} \scorematrix_{\text{\textup{sort}}}^\transpose \Delta_{\pi\omega} = \frac{1}{(\size - 2)! (\size H_{\size, 2} - H_{\size, 1}^2)} \sum_{p = 1}^{\size} \left( \frac{1}{\pi(p)} - \frac{1}{\omega(p)} \right)^2$.

The condition number of~$\scorematrix_{\text{\textup{sort}}}$ equals the square root of the ratio between the maximal and minimal eigenvalues of~$\scorematrix_{\text{\textup{sort}}}^\transpose \scorematrix_{\text{\textup{sort}}}$.
Subtracting~$(\size-2)! (\size H_{\size, 2} - H_{\size, 1}^2) \id_\size$ from~$\scorematrix_{\text{\textup{sort}}}^\transpose \scorematrix_{\text{\textup{sort}}}$ we get a matrix of rank~1, which means that $\size - 1$ eigenvalues of~$\scorematrix_{\text{\textup{sort}}}^\transpose \scorematrix_{\text{\textup{sort}}}$ equal $(\size-2)! (\size H_{\size, 2} - H_{\size, 1}^2)$.
The remaining eigenvalue corresponds to the eigenvector ${\bf 1}$ and equals $(\size - 1)! H_{\size, 1}^2$. With these eigenvalues we get the condition number~$\condnum(\scorematrix_{\text{\textup{sort}}}) = \tfrac{\sqrt{\size - 1} H_{\size, 1}}{\sqrt{\size H_{\size, 2} - H_{\size, 1}^2}}$.
\end{proof}

\begin{lemma}\label{th:rankingxijv}
The matrix $\lossmatrix_{\text{\textup{mAP}}}^\transpose \scorematrix_{\text{\textup{sort}}}$ has the following form:
\begin{equation}
\Bigl( \lossmatrix_{\text{\textup{mAP}}}^\transpose \scorematrix_{\text{\textup{sort}}} \Bigr)_{\gtvar, p} = \begin{cases}
\alpha(|\gtvarv|), & \gtvar_p = 1, \\
\beta(|\gtvarv|), & \gtvar_p = 0.
\end{cases}
\end{equation}
That is, for each ground-truth value $\gtvarv$ the matrix row components have only two values that depend on the Hamming norm $|\gtvarv| : = \sum_{p=1}^{\size} \gtvar_p$, specifically:
\begin{align}
\alpha(|\gtvarv|) = 
& \mathfrak{A}_{\size} \left(1 - \frac{|\gtvarv| - 1}{\size - 2} \left(1 - \frac{\size}{|\gtvarv| (\size - 1)}\right)\right)
- \mathfrak{B}_{\size} \left(\frac{3}{2} \frac{|\gtvarv| - 1}{\size - 2} \frac{\size - |\gtvarv|}{|\gtvarv|} \right)
- \mathfrak{C}_{\size} \left(\frac{\size - |\gtvarv|}{|\gtvarv| (\size - 1)} \right), \\
\beta(|\gtvarv|) = 
& \mathfrak{A}_{\size} \left(1 - \frac{|\gtvarv| - 1}{\size - 2} \right)
- \mathfrak{B}_{\size} \left(1 - \frac{3}{2} \frac{|\gtvarv| - 1}{\size - 2} \right).
\end{align}
Here $\mathfrak{A}_{\size} = (\size - 1)! H_{\size, 1}$, $\mathfrak{B}_{\size} = (\size - 2)!(H_{\size, 1}^2 - H_{\size, 2})$, $\mathfrak{C}_{\size} = (\size - 1)! H_{\size, 2}$. As a result, for permutations $\pi$ and $\omega$, we obtain
\begin{equation}
\bigl(\lossmatrix_{\text{\textup{mAP}}}^\transpose \proj_{\scoresubset_{\text{\textup{sort}}}} \Delta_{\pi \omega}\bigr)_\gtvarv = \gamma(|\gtvarv|) \left((\scorematrix_{\text{\textup{sort}}} \gtvarv)_\pi - (\scorematrix_{\text{\textup{sort}}} \gtvarv)_\omega \right),
\end{equation}
where $\gamma(p) = \frac{\alpha(p) - \beta(p)}{(\size - 2)!\left(\size H_{\size, 2} - H_{\size, 1}^2 \right)}$.
\end{lemma}

\begin{proof}
For brevity, here we denote~$\scorematrix_{\text{\textup{sort}}}$ by~$\scorematrix$, $\scoresubset_{\text{\textup{sort}}}$ by~$\scoresubset$ and $\lossmatrix_{\text{\textup{mAP}}}$ by~$\lossmatrix$.
Following the definitions of $\lossmatrix$ and $\scorematrix$, we explicitly compute the components of $\lossmatrix^\transpose \scorematrix$:
\begin{equation}
\label{eq:proof_mAP_LF}
\bigl( \lossmatrix^\transpose \scorematrix \bigr)_{\gtvarv, s} = \sum_{\sigma \in S_\size} \left(1 - \frac{1}{\gtvarv}\sum_{p=1}^{\size} \sum_{q=1}^{p} \frac{\gtvar_{\sigma^{-1}(p)}\gtvar_{\sigma^{-1}(q)}}{p}\right) \frac{1}{\sigma(s)}.
\end{equation}

There are exactly $(r - 1)!$ permutations with one fixed element, so we have $\sum_{\sigma \in S_\size} \frac{1}{\sigma(s)} = (\size - 1)! \sum_{p=1}^\size \frac{1}{p} = (r - 1)! H_{\size, 1} =: \mathfrak{A}_{\size}$. To compute the remaining part of~\eqref{eq:proof_mAP_LF},  we group the permutation values $\sigma(k) = t$ by each $t=1, \dots, \size$ and move the sum over permutations inside the bracket:
\begin{equation}
-\frac{1}{\gtvarv}\sum_{\sigma \in S_\size}\sum_{p=1}^{\size} \sum_{q=1}^{p} \frac{\gtvar_{\sigma^{-1}(p)}\gtvar_{\sigma^{-1}(q)}}{p \sigma(s)} = -\frac{1}{\gtvarv} \sum_{t=1}^\size \sum_{p=1}^{\size}\sum_{q=1}^{p} \sum_{\sigma \in S_\size, \sigma(s) = t} \frac{\gtvar_{\sigma^{-1}(p)}\gtvar_{\sigma^{-1}(q)}}{pt}.
\end{equation}
Next, we compute the inner sum $\sum_{\sigma \in S_\size, \sigma(s) = t} \gtvar_{\sigma^{-1}(p)}\gtvar_{\sigma^{-1}(q)}$. We rewrite the sum as the sum over inverse permutations:
\begin{equation}
\sum_{\sigma \in S_\size, \sigma(s) = t} \gtvar_{\sigma^{-1}(p)}\gtvar_{\sigma^{-1}(q)} = \sum_{\pi \in S_\size, \pi(t) = s} \gtvar_{\pi(p)}\gtvar_{\pi(q)}.
\end{equation}

The number of positive terms is different for the two cases of $y_s = 0$ and $y_s = 1$. For $y_s = 0$, using the Iverson brackets the sum can be rewritten as follows:
\begin{equation}
\sum_{\pi \in S_\size, \pi(t) = s} \gtvar_{\pi(p)}\gtvar_{\pi(q)} = \lbrack p \neq t \rbrack \left( \lbrack q < p \;\&\; q \neq t \rbrack |\gtvarv| (|\gtvarv| - 1) (\size - 3)!  + \lbrack q = p \rbrack |\gtvarv| (\size - 2)! \right).
\end{equation}
We then sum the expression over $q = 1, \dots, p$:
\begin{align}
\sum_{q=1}^p \lbrack p \neq t \rbrack \left( \lbrack q < p, q \neq t \rbrack |\gtvarv| (|\gtvarv| - 1) (\size - 3)!  + \lbrack q = p \rbrack |\gtvarv| (\size - 2)! \right) = \\[-2mm]
\lbrack p \neq t \rbrack \left((p - 1 - \lbrack t < p \rbrack) |\gtvarv| (|\gtvarv| - 1) (\size - 3)! + |\gtvarv| (r - 2)! \right).
\end{align}
Finally, we multiply the expression by $\frac{-1}{|\gtvarv|pt}$ and compute the sum over $p$ and $t$:
\begin{align}
-\frac{1}{|\gtvarv|}\sum_{t=1}^{\size}\sum_{p=1}^{\size} \frac{1}{tp} \lbrack p \neq t \rbrack \left((p - 1 - \lbrack t < p \rbrack) |\gtvarv| (|\gtvarv| - 1) (\size - 3)! + |\gtvarv| (r - 2)! \right) = \\[-2mm]
-\mathfrak{A}_{\size} \frac{|\gtvarv| - 1}{\size - 2} + \mathfrak{B}_{\size}\left(1 - \frac{3}{2}\frac{|\gtvarv| - 1}{\size - 2} \right).
\end{align}
Combining with $\sum_{\sigma \in S_\size} \frac{1}{\sigma(s)} = \mathfrak{A}_{\size}$ we obtain the desired expression for $\beta(|\gtvarv|)$.

Now, consider the case of $y_s = 1$. Again, we rewrite the sum as
\begin{align}
\sum_{\pi \in S_\size, \pi(t) = s} \gtvar_{\pi(p)}\gtvar_{\pi(q)} & =
\lbrack p = t \rbrack \left(\lbrack q < p \rbrack (|\gtvarv| - 1) (\size - 2)! + \lbrack q = p \rbrack (\size - 1)! \right) \\[-4mm]
& + \lbrack p \neq t \rbrack \left( \lbrack q < p \;\&\; q = t \rbrack (|\gtvarv| - 1) (\size - 2)! \right) \\
& + \lbrack p \neq t \rbrack \left(\lbrack q < p \;\&\; q \neq t \rbrack (|\gtvarv| - 1)(|\gtvarv| - 2) (\size - 3)!\right) \\
& + \lbrack p \neq t \rbrack \left( \lbrack q = p \rbrack (|\gtvarv| - 1) (\size - 2)! \right),
\end{align}
sum it over $q$,
\begin{align}
\sum_{q=1}^{p} \sum_{\pi \in S_\size, \pi(t) = s} \gtvar_{\pi(p)}\gtvar_{\pi(q)} &= \lbrack p = t \rbrack \left((p - 1) (|\gtvarv| - 1) (\size - 2)! + (\size - 1)! \right) \\[-4mm]
& + \lbrack p \neq t \rbrack \left( \lbrack t < p \rbrack (|\gtvarv| - 1) (\size - 2)! \right) \\
& + \lbrack p \neq t \rbrack \left( (p - 1 - \lbrack t < q \rbrack )(|\gtvarv| - 1) (|\gtvarv| - 2) (\size - 3)! \right) \\
& + \lbrack p \neq t \rbrack (|\gtvarv| - 1) (\size - 2)!,
\end{align}
and obtain the desired expression by multiplying the latter by $\frac{-1}{|\gtvarv| p t}$ and summing over $p$ and $t$. The last step of the computation is completely analogous to the case $y_k=0$ and we omit it for brevity.

To compute $\bigl(\lossmatrix^T \proj_\scoresubset \Delta_{\pi \omega}\bigr)_\gtvarv$, we note that for all permutations $\pi$ and $\omega$ it holds that ${\bf 1}^\transpose \scorematrix^\transpose \Delta_{\pi \omega} = 0$. According to Lemma~\ref{th:rankingprojection} and the Sherman-Woodbury formula, $(\scorematrix^\transpose \scorematrix)^{-1}$ is the sum of the scalar matrix $\frac{\id_\size}{(\size - 2)! (\size H_{\size, 2} - H_{\size, 1}^2)}$ and the multiple of the rank one matrix ${\bf 1}{\bf 1}^\transpose$. After the multiplication on~$ \scorematrix^\transpose \Delta_{\pi \omega}$, the second term vanishes, so we get $(\scorematrix^\transpose \scorematrix)^{-1} \scorematrix^\transpose \Delta_{\pi \omega} = \frac{\sum_{p=1}^{\size} \frac{1}{\pi(p)} - \frac{1}{\omega(p)}}{(\size - 2)! (\size H_{\size, 2} - H_{\size, 1}^2)}$. Finally, we rewrite $\bigr(\lossmatrix^T \scorematrix\bigl)_{\gtvarv, :}$  as $(\alpha(|\gtvarv|) - \beta(|\gtvarv|)) \gtvar + \beta(|\gtvarv|) {\bf 1}$. By the same argument, after the vector multiplication, the second component vanishes and we get $\bigl(\lossmatrix^\transpose \scorematrix \left(\scorematrix^\transpose \scorematrix \right)^{-1} \scorematrix \Delta_{\pi \omega} \bigr)_\gtvarv = \frac{(\scorematrix \gtvarv)\pi - (\scorematrix \gtvarv)\omega}{(\size - 2)! (\size H_{\size, 2} - H_{\size, 1}^2)}$, which finishes the proof.
\end{proof}

\begin{lemma}\label{th:rankingasymptotics}
For the score set $\scoresubset_{\text{\textup{sort}}}$, we have $2 (\size - 1)! \| \proj_{\scoresubset_{\text{\textup{sort}}}} \Delta_{\pi \omega} \|_2^2 = O(\size)$.
We also have that $\gamma(|\gtvarv|)$ defined in Lemma~\ref{th:rankingxijv} with $|\gtvarv| = \lambda \size, \lambda \in (0, 1)$ vanishes as $\size$ approaches infinity: $\gamma(|\gtvarv|) = O(\frac{\log^2 \size}{\size})$.
The condition number~$\condnum(\scoresubset_{\text{\textup{sort}}})$ grows as~$\Theta(\log \size)$.
\end{lemma}
\begin{proof}
To derive an asymptotic bound for $\| \proj_{\scoresubset_{\text{\textup{sort}}}} \Delta_{\pi \omega} \|_2^2$, we elaborate on the sum of squares $\sum_{p=1}^{\size} \left( \frac{1}{\pi(p)} - \frac{1}{\omega(p)} \right)^2 = 2 H_{\size, 2} - 2\sum_{p=1}^{\size} \frac{1}{\pi(p) \omega(p)} \leq 2H_{\size, 2}$ and apply the asymptotic bounds for the harmonic numbers $H_{\size, 1}^2 = \Theta(\log^2 \size), H_{\size, 2} = \Theta(1)$:
\begin{equation}
2 (\size - 1)! \| \proj_{\scoresubset_{\text{\textup{sort}}}} \Delta_{\pi \omega} \|_2^2 = O(\frac{(\size)!}{(\size - 2)! \size}) = O(\size)
\end{equation}
For the second part of the lemma, we rewrite $\alpha(|\gtvarv|)$ and $\beta(|\gtvarv|)$:
\begin{align}
\alpha(|\gtvarv|) &= \mathfrak{A}_{\size} \left( 1 - \lambda(1 - \frac{1}{\size}) \right) - \mathfrak{B}_{\size} \frac{3}{2} (1 - \lambda) - \mathfrak{C}_{\size} \frac{1 - \lambda}{\size} + o(1) \\
\beta(|\gtvarv|) &= \mathfrak{A}_{\size} \left( 1 - \lambda \right) - \mathfrak{B}_{\size} \left(1 - \frac{3}{2} \lambda \right) + o(1) \\
\alpha(|\gtvarv|) - \beta(|\gtvarv|) &= \mathfrak{A}_{\size} \frac{1}{\size} - \mathfrak{B}_{\size} \frac{1}{2} - \mathfrak{C}_{\size} \frac{1 - \lambda}{\size} + o(1)
\end{align}
By definition, we have $\mathfrak{A}_{\size} = \Theta((\size - 1)! \log \size)$, $\mathfrak{B}_{\size} = \Theta((\size - 2)! \log^2 \size)$, $\mathfrak{C}_{\size} = \Theta((\size - 1)!)$, which gives us
\begin{equation}
\gamma(|\gtvarv|) = \frac{\alpha(|\gtvarv|) - \beta(|\gtvarv|)}{(\size - 2)! (\size H_{\size, 2} - H_{\size, 1}^2)} = O\left(\frac{(\size - 2)! \log^{2} \size}{(\size - 1)!}\right) = O(\frac{\log^2 \size}{\size}),
\end{equation}
what was to be shown.

Finally, the asymptotic bound for the condition number of~$\scoresubset_{\text{\textup{sort}}}$ trivially follows from its exact expression in Lemma~\ref{th:rankingprojection} and the asymptotic bounds for the harmonic numbers.
\end{proof}

\end{document}